\PassOptionsToPackage{svgnames}{xcolor}
\documentclass{article} 

\usepackage[utf8]{inputenc} 
\usepackage[T1]{fontenc}    
\usepackage{url}            
\usepackage{booktabs}       
\usepackage{amsfonts}       
\usepackage{nicefrac}       
\usepackage{enumitem}

\usepackage{amsmath}
\usepackage{url}
\usepackage[bitstream-charter]{mathdesign}
\usepackage[
  paper  = letterpaper,
  margin = 1in,
  ]{geometry}

\usepackage{amsthm,amsmath}

\usepackage[english]{babel}
\usepackage{afterpage}
\usepackage{framed}
\usepackage{scalerel}
\usepackage{subcaption}
\usepackage{float}

\usepackage{graphicx}
\usepackage[labelfont=bf]{caption}
\usepackage{booktabs}

\usepackage{listings}
\usepackage{fancyvrb}
\fvset{fontsize=\normalsize}

\usepackage[authoryear,sort]{natbib}
\usepackage[colorlinks,
            linkcolor = blue,
            urlcolor  = blue,
            citecolor = blue,
            linktoc=all]{hyperref}
\usepackage[nameinlink]{cleveref}

\usepackage[acronym,nowarn,nohypertypes={acronym,notation}]{glossaries}
\setacronymstyle{long-sc-short}

\lstdefinestyle{mystyle}{
    commentstyle=\color{OliveGreen},
    keywordstyle=\color{BurntOrange},
    numberstyle=\tiny\color{black!60},
    stringstyle=\color{MidnightBlue},
    basicstyle=\ttfamily,
    breakatwhitespace=false,
    breaklines=true,
    captionpos=b,
    keepspaces=true,
    numbers=left,
    numbersep=5pt,
    showspaces=false,
    showstringspaces=false,
    showtabs=false,
    tabsize=2
}
\usepackage{tikz}
\usetikzlibrary{shapes,decorations,arrows,calc,arrows.meta,fit,positioning}
\tikzset{
    -Latex,auto,node distance =1 cm and 1 cm,semithick,
    state/.style ={circle, draw, minimum width = 0.7 cm},
    detstate/.style ={rectangle, draw, minimum width = 0.7 cm, minimum height = 0.7 cm},
    point/.style = {circle, draw, inner sep=0.04cm,fill,node contents={}},
    bidirected/.style={Latex-Latex,dashed},
    el/.style = {inner sep=2pt, align=left, sloped}
}

\usepackage{wrapfig}
\usepackage{mathtools}

\usepackage{mfirstuc}
\usepackage{xspace}

\DeclareRobustCommand{\mb}[1]{\ensuremath{\boldsymbol{\mathbf{#1}}}}

\renewcommand{\mid}{~\vert~}

\newcommand{\mba}{\mb{a}}
\newcommand{\mbb}{\mb{b}}

\newcommand{\mbe}{\mb{e}}

\newcommand{\mbx}{\mb{x}}
\newcommand{\mby}{\mb{y}}
\newcommand{\mbz}{\mb{z}}

\newcommand{\mbI}{\mb{I}}

\newcommand{\mbR}{\mb{R}}

\newcommand{\mbX}{\mb{X}}

\newcommand{\mbdelta}{\mb{\delta}}
\newcommand{\mbepsilon}{\mb{\epsilon}}

\newcommand{\cF}{\mathcal{F}}

\newcommand{\cN}{\mathcal{N}}

\newcommand{\E}{\mathbb{E}}

\newcommand{\cU}{\mathcal{U}}

\newcommand{\g}{\mid}

\newtheorem{prop}{Proposition}
\newtheorem{thmdef}{Definition}

\newtheorem{thm}{Theorem}

\newtheorem*{lemma*}{Lemma}

\crefname{lemma}{lemma}{lemmas}
\crefname{prop}{proposition}{propositions}

\newcommand{\indep}{\rotatebox[origin=c]{90}{$\models$}}
\newcommand{\nindep}{\rotatebox[origin=c]{90}{$\not\models$}}

\newcommand{\mbeps}{\mbepsilon}

\newcommand{\ind}{\mathbf{1}}

\newcommand{\ptr}{{p_{tr}}}

\newcommand{\pte}{{p_{te}}}
\newcommand{\pind}{{p_{\scaleto{\indep}{4pt}}}}

\newcommand{\pd}{{p_D}}

\newacronym{KL}{kl}{Kullback-Leibler}
\newacronym{ELBO}{elbo}{\emph{evidence lower bound}}
\newacronym{POPELBO}{pop-elbo}{\emph{population evidence lower bound}}

\newacronym{SVI}{svi}{stochastic variational inference}
\newacronym{BUMPVI}{bump-vi}{bumping variational inference}

\newacronym{GMM}{gmm}{Gaussian mixture model}
\newacronym{LDA}{lda}{latent Dirichlet allocation}

\newacronym{SUTVA}{sutva}{stable unit treatment value assumption}

\newacronym{KSD}{ksd}{{kernelized Stein discrepancy}}
\newacronym{KCC-SD}{kcc-sd}{kernelized complete conditional Stein discrepancy}

\newacronym{OPVI}{opvi}{operator variational inference}
\newacronym{SVGD}{svgd}{Stein variational gradient descent}

\newacronym{erm}{erm}{empirical risk minimization}

\newacronym{nurd}{nurd}{Nuisance-Randomized Distillation}
\newacronym{jtt}{jtt}{Just Train Twice}
\newacronym{lff}{lff}{Learning from Failure}
\newacronym{nli}{nli}{natural language inference}
\newacronym{poe}{poe}{product of experts}
\newacronym{dfl}{dfl}{debiased focus loss}

\newacronym{pr}{patch-rnd}{patch randomization}
\newacronym{nr}{ngram-rnd}{n-gram randomization}

\newacronym{ood}{ood}{out-of-distribution}

\newacronym{scam}{b-scam}{biased-model-based spurious-correlation avoiding method}

\newcommand{\nmeth}{\gls{scam}}

\newcommand{\nmeths}{\glspl{scam}}
\newcommand{\Nmeths}{\Glspl{scam}}

\newacronym{pm}{prem-mask}{premise masking}
\newacronym{roi}{roi}{region-of-interest}
\newacronym{rm}{roi-mask}{\textsc{roi} masking}

\newacronym{ff}{freq-filt}{frequency filtering}

\newacronym{if}{int-filt}{intensity filtering}

\newacronym{cad}{cad}{counterfactually augmented data}

\newacronym{dgp}{dgp}{data generating process}
\newcommand{\nrd}{nuisance-randomized distribution\xspace}

\newcommand{\nlr}{nuisance-label relationship}

\usepackage{titlesec}
\usepackage{soul}
\usepackage{multirow}
\usepackage{algorithm}
\usepackage{etoolbox}

\usepackage{algorithmic}
\usepackage{titlesec}

\makeatletter
\renewcommand{\paragraph}{%
\@startsection{paragraph}{4}%
{\z@}{1ex \@plus 0.0ex \@minus 0.0ex}{-1em}%
{\normalfont\normalsize\bfseries}%
}
\titlespacing*{\section}
{0pt}{1.5ex plus 0.1 ex minus .1ex}{1.5ex plus .1ex}
\titlespacing*{\subsection}
{0pt}{1.1ex plus 0.1ex minus 0.1ex}{1.1ex plus .2ex}
\makeatother

\title{\textbf{Nuisances via Negativa:\\ Adjusting for Spurious Correlations via Data Augmentation}}

\author{
 Aahlad Puli\textsuperscript{1}\thanks{\textsuperscript{1}Corresponding author: \href{mailto:aahlad@nyu.edu}{aahlad@nyu.edu}. Published at TMLR 2024: \url{https://openreview.net/forum?id=RIFJsSzwKY}.}
  \hspace{13pt} 
  Nitish Joshi \textsuperscript{1} \hspace{13pt}
  Yoav Wald  \textsuperscript{2}\hspace{13pt} 
  He He\textsuperscript{1,2}\hspace{13pt} 
  Rajesh Ranganath\textsuperscript{1,2,3}
     \\\\
     \hspace{-10pt} 
     \textsuperscript{1}Department of Computer Science, New York University  \\
     \hspace{-10pt} 
     \textsuperscript{2}Center for Data Science, New York University \\
     \hspace{-10pt} 
     \textsuperscript{3}Department of Population Health, Langone Health, New York University
}

\def\month{06}  
\def\year{2024} 
\date{}

\begin{document}

\maketitle

\begin{abstract}
\noindent In prediction tasks, there exist features that are related to the label in the same way across different settings for that task; these are semantic features or semantics. 
Features with varying relationships to the label are nuisances. 
For example, in detecting cows from natural images, the shape of the head is semantic but because images of cows often have grass backgrounds but not always, the background is a nuisance. 
Models that exploit nuisance-label relationships face performance degradation when these relationships change. 
Building models robust to such changes requires additional knowledge beyond samples of the features and labels. 
For example, existing work uses annotations of nuisances or assumes \acrshort{erm}-trained models depend on nuisances. 
Approaches to integrate new kinds of additional knowledge enlarge the settings where robust models can be built. 
We develop an approach to use knowledge about the semantics by corrupting them in data, and then using the corrupted data to produce models which identify correlations between nuisances and the label.
Once these correlations are identified, they can be used to adjust for where nuisances drive predictions.
We study semantic corruptions in powering different spurious-correlation avoiding methods on multiple \gls{ood} tasks like classifying waterbirds, \gls{nli}, and detecting cardiomegaly in chest X-rays.
\end{abstract}

\section{Introduction}

Relationships between the label and the covariates can change across data collected at different places and times.
For example, in classifying animals, data collected in natural habitats have cows appear more often on grasslands, while penguins appear more often on backgrounds of snow; these animal-background relationships do not hold outside natural habitats \citep{beery2018recognition,arjovsky2019invariant}.
Some features, like an animal's shape, are predictive of the label across all settings for a task; these are \textit{{semantic features}}, or \textit{{semantics}} in short.
Other features with varying relationships with the label, like the background, are {nuisances}.
Even with semantics present, models trained via \gls{erm} can predict using nuisances and thus fail to generalize~\citep{geirhos2020shortcut}. 
Models that rely only on the semantic features perform well even when the \nlr{} changes, unlike models that rely on nuisances.

\glsreset{ood}

Building models that generalize under changing nuisance-label relationships requires additional knowledge, beyond a dataset of features and labels sampled from the training distribution.
For example, many works assume knowledge of the nuisance. In the animal-background example, this would correspond to a feature that specifies the image background, which we may use when specifying our learning algorithm. \citep{mahabadi2019end,makar2021causally,veitch2021counterfactual,puli2021predictive}; another common type of assumption is access to multiple datasets over which the nuisance-label correlation varies \citep{peters2016causal, arjovsky2019invariant, wald2021calibration}, and some other forms of knowledge have been explored \citep{mahajan2021domain, gao2023out, Feder2023DataAF}.

\textbf{Semantic Corruptions.} In this paper, we explore the use of a different type of knowledge: corruptions of semantic features. Intuitively, imagine trying to predict the label from a corrupted input $T(\mbx)$, where all semantic information has been removed. Any better-than-chance prediction provides us a window into the nuisances, as it must rely on them. We will then use these obtained biased models to guide methods that we identify here as \glspl{scam}.

\textbf{\Glspl{scam}.} There is a class of methods in the literature that use predictions of a biased model to adjust for nuisances, and learn predictors that are free of spurious correlations. Among others, these include \gls{jtt} \citep{liu2021just}, EILL \citep{creager2021environment}, \gls{nurd} \citep{puli2021predictive}, and \gls{dfl}, \gls{poe} \citep{mahabadi2019end}. 
The key question arising from these works is \emph{how can we obtain biased models?} In empirical studies, prior works on \glspl{scam} either use annotations of the nuisance or an ERM-trained model over the training data as a placeholder for the biased model. The latter approach, based on an ERM-trained model, is successful if that model completely ignores semantic information. In practice, these heuristics are rather fragile. Annotations for nuisances are seldom available, and we lack a principled method to ascertain whether a model trained with \gls{erm} relies only on semantic features. Therefore, employing semantic corruptions could serve as a valuable alternative to these heuristics. We claim that semantic corruptions offer a principled and useful approach to obtaining biased models.

Semantic corruptions $T(\mbx)$ must strike a delicate balance between removing semantic information and preserving nuisances. 
For example, if $T(\mbx)$ replaces all pixels in an image with random noise, it corrupts semantics while simultaneously erasing all information about the nuisances.
An ideal $T(\mbx)$ would isolate nuisances by targeting only the semantic information in the input, e.g., by in-painting the animal for the task of classifying cows and penguins.
Implementing such ideal corruptions is unrealistic, as they are task-specific and may require human annotations of the semantic features; e.g., one can segment the objects in every image. 
Doing so for all classification problems is extremely laborious. 
In tasks like \gls{nli}, it is unclear even \emph{how} to annotate semantics, as they do not correspond to simple features like subsets of words. 
In summary, after outlining the desired characteristics of semantic corruptions, we define corruptions that are beneficial across multiple tasks and do not require human annotation. 
Our contributions are as follows:
\begin{enumerate}[itemsep=0pt,topsep=0pt,leftmargin=15pt,partopsep=5pt]
    \item 
Show that acquiring additional knowledge beyond a labeled dataset is necessary for effectively learning robust models (\cref{thm:assumptions}). Then, in proposition 1, we formalize sufficient conditions under which additional knowledge in the form of a semantic corruption enables \nmeths{} to learn robust models.
\item 
Develop multiple semantic corruptions for object recognition and natural language inference. These include patch randomization, n-gram randomization, frequency filtering, and intensity filtering. Then, we situate existing procedures, such as region-of-interest masking and premise masking, under the umbrella of semantic corruptions.
\item 
Empirically, we demonstrate that any semantic corruption can power any \nmeth{}. The corruption-powered versions of these methods outperform \gls{erm} on \gls{ood} generalization tasks like Waterbirds, cardiomegaly detection from chest X-rays, and NLI. Corruption-powered \gls{nurd}, \gls{dfl}, and \gls{poe} achieve performance similar to said methods run with extra observed nuisance variables. Corruption-powered \gls{jtt} outperforms vanilla \gls{jtt}.
\end{enumerate}

\glsreset{scam}
\section{\Acrlongpl{scam}}\label{sec:nmeths}

\looseness=-1
A spurious correlation is a relationship between the covariates $\mbx$ and the label $\mby$ that changes across settings like time and location \citep{geirhos2020shortcut}.
The features whose relationship with the label changes are called nuisances.
With a vector of nuisances $\mbz$, let $\ptr(\mby, \mbz, \mbx), \pte(\mby, \mbz, \mbx)$ be the training and test distributions.

\paragraph{Achieving robustness to spurious correlations requires additional knowledge.}
In the presence of spurious correlations, the training distribution $\ptr$ may not equal the test distribution $\pte$.
Without further assumptions, no algorithm that only sees data from $\ptr(\mby, \mbx)$ can produce a predictor that works well on $\pte$.
To achieve generalization when $\pte\not= \ptr$, work in the \gls{ood} generalization literature assumes a relationship between the training and test distributions.
We follow the work of \citet{makar2021causally,puli2021predictive} and assume that only nuisance-label relationships --- i.e. the conditional $\mbz\g\mby$ --- changes between training and test.
Formally,
 we let $\ptr, \pte$ come from a family of distributions whose members have different nuisance-label relationships but share the same relationship between the label and the semantics $\mbx^*$: 
\begin{thmdef}\label{def:nuisance_varying}(Nuisance-varying family with semantic features $\mbx^*$ \citep{makar2021causally,puli2021predictive})
\begin{align}
\label{eq:nvf}
\cF = \left\{p_D \,\, : \,\, p_D(\mby, \mbz, \mbx^*, \mbx) = p(\mby, \mbx^*)\,\, p_D(\mbz\g \mby)\,\, p(\mbx\g \mbz, \mbx^*) \right\}.
\end{align}
\end{thmdef}
Many common tasks in \gls{ood} generalization, including some from \cref{sec:exps}, fit this definition.
For example, in classifying natural images, the background type is the nuisance $\mbz$ and its relationship to the label can change across places, each corresponding to a different member of $\cF$.
The animal shape however is made of semantic features $\mbx^*$ that are related to the label in the same way across places.
Like in this example, we assume that the semantic features $\mbx^*$ equal a function of the covariates $\mbx^* = e(\mbx)$ almost surely under every $p_D\in \cF$, but neither $\mbx^*$ nor $e(\cdot)$ are known.
Finally, the semantics and nuisances together account for all the information that $\mbx$ has about $\mby$, meaning $\mbx \indep_{p_D} \mby \g \mbx^*, \mbz$.

Building models that are robust to a shifting nuisance-label relationship relies on additional knowledge, such as nuisance annotations, in the training data \citep{sagawa2019distributionally,veitch2021counterfactual,makar2021causally,puli2021predictive,yao2022improving}.
Given knowledge of $\mbz$, work like \citep{makar2021causally,puli2021predictive}
estimate a distribution, denoted $\pind$, under which the label and nuisance are independent ($\mby\indep_{\pind} \mbz$): $\pind(\mby, \mbx)  = \int_{z, x^*} p(\mby, \mbx^*=x^*) \ptr(\mbz=z) p(\mbx\g \mbz=z, \mbx^*=x^*) dz dx^*.$
Following \cite{puli2021predictive}, we call $\pind$ the \textit{nuisance-randomized distribution.}
The model $\pind(\mby=1\g \mbx)$ achieves the lowest risk on 
any member of the family $\cF$ amongst the set of risk-invariant models; see Proposition 1 \citep{makar2021causally}).
However, even when $\ptr,\pte\in \cF$ and optimal risk-invariant predictors can be built with nuisances, 
\textit{it is impossible to always beat random chance when given data $\{\mby,\mbx\} \sim \ptr$:}
\begin{thm}\label{thm:assumptions}
For any learning algorithm,
there exists a nuisance-varying family $\cF$ where predicting with $\pind(\mby=1 \g \mbx)$ achieves $90\%$ accuracy on all members 
such that given training data $\mby, \mbx$ from one member $\ptr\in \cF$, 
the algorithm cannot achieve better accuracy than $50\%$ (random chance) on some $\pte\in \cF$.
\end{thm}
\looseness=-1
The proof is in \cref{appsec:thm} and proceeds in two steps.
With $\text{ACC}_{h}(p)$ as the expected accuracy of a model $h$ on distribution $p$, the first step of the proof defines two nuisance-varying families $\cF_1, \cF_2$ such that no single model can perform well on both families simultaneously;
any $h(\mbx)$ for which $\text{ACC}_{p_1}(h) > 50\%$ for all $p_1\in \cF$ will have that $\text{ACC}_{p_2}(h) < 50\%$ for some $p_2\in{\cF_2}$ and vice-versa.
The second step shows that the two families $\cF_1, \cF_2$ have a member that has the same distribution over $\mby,\mbx$; letting the training data come from this distribution means that any learning algorithm that returns a performant model --- one that beats $50\%$ accuracy --  on one family must fail to return a performant model on the other.
Next, we discuss different methods that use additional knowledge beyond $\mby,\mbx$ to build robust predictors.

\subsection{\Acrlongpl{scam}.}
We focus on methods that correct models using knowledge of nuisances or where they might appear in the covariates \citep{mahabadi2019end,liu2021just,puli2021predictive}.
We first establish that the common central part in these methods is a model that predicts the label using nuisances, which we call the \textit{{biased model}}; due to this commonality, we call these {\glsreset{scam}\nmeths{}}.
At a high level, a \nmeth{} has two components.
The first is a biased model that is built to predict the label by exploiting the nuisance-label relationship via extra knowledge or assumptions.
The biased model is then used to guide a second model to predict the label without relying on nuisances.

We briefly summarize the different \nmeths{} here, differentiated by the additional knowledge they use to build biased models.
The differences between the methods are summarized in \cref{tab:nmeth-summary}.
We give details for \gls{nurd} here and defer algorithmic details about the rest to \cref{appsec:background}.

\paragraph{Biased models from knowledge of the nuisances.}
The first category of \nmeths{} from \citet{mahabadi2019end,puli2021predictive} \textit{assumes additional knowledge in the form of nuisance annotations $\mbz$}.
For example, in \gls{nli} ---  where the goal is determining if a premise sentence entails a hypothesis ---  
\citep{mahabadi2019end} compute the fraction of words shared between the hypothesis and the premise for each sample in the training data and use this as one of the nuisance features in building the biased model.
The biased model in \gls{nurd}, \gls{poe}, \gls{dfl} is learned by predicting the label from the nuisance annotations in the training data to estimate $p_{tr}(\mby \g \mbz)$.
Using nuisance annotations, \citet{puli2021predictive, makar2021causally} use the model $p_{tr}(\mby \g \mbz)$ as the biased model to define importance weights and minimize risk w.r.t a distribution $\pind$ obtained as
$$ \pind(\mby , \mbz, \mbx)= \ptr(\mby) \ptr(\mbz) p(\mbx\g \mby, \mbz) = \frac{ p(\mby) }{ \ptr( \mby \mid \mbz )  } \ptr(\mbz) \ptr(\mby\g \mbz) p(\mbx\g \mby, \mbz)  =  \frac{ p(\mby) }{ \ptr( \mby \mid \mbz )  } \ptr (\mby, \mbz, \mbx) . $$
The second step in \gls{nurd} \citep{puli2021predictive} trains a model to predict $\mby$ from a representation $r(\mbx)$ on data from $\pind$ such that $\mbz\indep_\pind \mby\g r(\mbx)$; this step is called distillation.
Due to $\mby\indep_\pind \mbz$, learning in $\pind$ avoids features that depend only on the nuisance and due to $\mbz\indep_\pind \mby\g r(\mbx)$, distillation avoids features that are mixed functions of the label and the nuisance  (e.g. $\mbx_1=\mby + \mbz$).
With these insights, \gls{nurd} builds models of the form $\pind(\mby\g r(\mbx))$ that are most informative of the label.
Mechanically, \gls{nurd}'s distillation solves this:
\[\max_{\theta,\gamma}  \mathbf{E}_{ \pind}  \log p_\theta( \mby \mid r_\gamma(\mbx) ) - \lambda \mathbf{I}_{ \pind } (\mby; \mbz \mid r_\gamma(\mbx)).\]
\citet{puli2021predictive} show that such models are best in a class of predictors with lower bounds on performance.
The mutual information above is zero when $\mby \indep_\pind \mbz \g \mbx$; this condition holds for semantic corruptions as we discuss in \cref{appsec:background}.
Thus, we run the distillation step as importance-weighted \gls{erm} on the training data.

\citet{mahabadi2019end} consider two methods to train a {biased} model and a base predictive model jointly to make the base model predict without relying on the biases.
They propose 1) \gls{poe}, where the loss is the sum of the \texttt{log} loss of the two models and 2) \gls{dfl}, where the biased model is used to weight the cross-entropy loss for the base model.
For both methods, \citet{mahabadi2019end} build a biased model as $\ptr(\mby\g \mbz)$.
Intuitively, the base model focuses on classifying samples that the biased model misclassifies.
The methods fine-tune a BERT model \citep{Devlin2019BERTPO} and do not propagate the gradients of the biased model to update the common parameters (token embeddings).

\begin{table*}[t]
\centering
 \begin{small}
    \caption{Summary of \gls{nurd}, \gls{jtt}, \gls{poe}, and \gls{dfl}.
Each method approximates the {biased model}: $\ptr(\mby\g \mbz)$.
This table describes the different biased models, their names, how they are built.
  }
    \label{tab:nmeth-summary}
  \centering
    \begin{tabular}{p{1.3cm}p{2.9cm}p{4cm}p{4cm}}
    \toprule
          Method &  Name & What the biased model is & Assumptions/Knowledge
      \\
    \midrule
       \gls{jtt}
            &  
            Identification model 
            & 
            $\ptr(\mby\g \mbx)$ learned via \gls{erm}  
            & 
            \gls{erm} learns biased models. 
      \\
            \gls{poe}/\gls{dfl} 
            &  
            Biased model 
            & 
            $\ptr(\mby\g \mbz)$ learned via \gls{erm}  
            & 
			$\mbz$ from domain-knowledge. 
      \\
            \gls{nurd}
            &  
            Weight model 
            & 
            $\ptr(\mby\g \mbz)$ learned via \gls{erm}  
            & 
			$\mbz$ from domain-knowledge. 
      \\
      \bottomrule
\end{tabular}
\end{small}
\end{table*}

\paragraph{Biased models from assumptions on \gls{erm}-trained models.}
The second category of \nmeths{} like LFF \citep{nam2020learning}, UMIX \citep{han2022umix}, and \gls{jtt} \citep{liu2021just} require \textit{additional knowledge that vanilla \gls{erm} builds a biased model that exploits the nuisance-label relationship}.
Given such a model, these works use it to reduce a second model's dependence on the nuisance.
We focus on \gls{jtt}~\citep{liu2021just} which aims to build models robust to group shift, where the relative mass of a fixed set of disjoint groups of the data changes between training and test times.
The groups here are subsets of the data defined by a pair of values of discrete label and nuisance values.
While \gls{jtt} works without relying on training group annotations, i.e. without nuisances, it assumes \gls{erm}'s missclassifications are because of a reliance on the nuisance.
\gls{jtt} first builds an ``identification'' model via \gls{erm} to isolate samples that are misclassified.
Then, \gls{jtt} trains a model via \gls{erm} on data with the loss for the misclassified samples upweighted (by constant $\lambda$).
The epochs to train the identification model and the upweighting constant are hyperparameters that require tuning using group annotations \citep{liu2021just}.

\paragraph{The commonality of a biased model.}
The central part in \gls{nurd}, \gls{dfl}, \gls{poe}, and \gls{jtt} is a model that predicts the label using nuisances, like $\ptr(\mby \g\mbz)$, which we call the {biased model} as in \citet{he2019unlearn}.
The predictive models in each \gls{scam} are guided to not depend on nuisances used by the biased model.
While \glspl{scam} reduce dependence on nuisances, they build biased models using additional nuisance annotations or require assumptions that \gls{erm}-trained models predict using the nuisance.
In the next section, we describe an alternative: corrupt semantic information with data augmentations to construct biased models.

\section{Out-of-distribution generalization via Semantic Corruptions} \label{sec:semantic_corruptions}
The previous section summarized how biased models can be built in \nmeths{} using either direct knowledge of nuisances or knowledge that \gls{erm}-trained models rely on the nuisances.
We now introduce semantic corruptions and show how they enable building biased models.
Semantic corruptions are transformations of the covariates that do not retain any knowledge of the semantics, except what may be in the nuisance $\mbz$:
\begin{thmdef}[Semantic Corruption]
\label{def:semantic_corruption}
A semantic corruption is a transformation of the covariates $T(\mbx, \mbdelta)$, where $\mbdelta$ is a random variable such that $\mbdelta \indep (\mby, \mbz,\mbx, \mbx^*)$, if
    \[\forall \, p_D\in \cF \quad T(\mbx, \mbdelta) \indep_{p_D} \mbx^* \g \mbz.\]
\end{thmdef}

Here, we characterize conditions under which biased models built from semantic corruptions could be used to estimate robust models.
As discussed in \cref{sec:nmeths}, $\pind(\mby\g \mbx)$ is the optimal risk-invariant predictor, and is the target of \gls{erm} when predicting the label $\mby$ from $\mbx$ under the nuisance-randomized distribution $\pind$.
\Gls{nurd} estimates this distribution as part of the algorithm, and methods like \gls{jtt} aim to approximate $\pind$, for example, upweighting samples mis-classified by a model that relies on $\mbz$ to predict $\mby$.
We compare $\pind$ which is obtained by breaking the nuisance-label relationship 
against the distribution obtained by breaking the relationship between the label and the data augmentation : 
\[\pind(\mby, \mbx) = \int_z \frac{\ptr(\mby)}{\ptr(\mby\g \mbz=z)}\ptr(\mby, \mbz=z,\mbx), \qquad\quad p_T(\mby, \mbx) = \int_\delta p(\mbdelta=\delta) \frac{\ptr(\mby)}{\ptr(\mby\g T(\mbx, \delta))}\ptr(\mby, \mbx) d \delta.\]
We show here that the $L_1$ distance between $\pind(\mby, \mbx)$ and $p_T(\mby, \mbx)$ is controlled by an $L_2$-distance between the biased models built from the nuisance and the data augmentations respectively:
\newcommand{\propone}{Let $T:\mbX\times \mbR^d \rightarrow \mbX$ be a function. Assume the r.v. ${\ptr(\mby\g T(\mbx, \mbdelta))}^{-1}$ has a bounded second moment under the distribution $\pind(\mby, \mbz, \mbx)p(\mbdelta)$, and that $\ptr(\mby\g T(\mbx, \mbdelta))$ and $\ptr(\mby \g \mbz)$ satisfy
 \[\E_{\pind(\mby, \mbz, \mbx)p(\mbdelta)}{\ptr(\mby\g T(\mbx, \mbdelta))^{-2}} \leq m^2,\quad  \quad \E_{\pind(\mby, \mbz, \mbx)p(\mbdelta)} \left|\ptr(\mby\g T(\mbx, \mbdelta)) - \ptr(\mby\g \mbz)\right|^2 = \epsilon^2.\]
Then, the $L_1$ distance between $\pind(\mby, \mbx)$ and $p_T(\mby, \mbx)$ is bounded: $\|\pind(\mby, \mbx) - p_T(\mby, \mbx)\|_{1} \leq m\epsilon$. For a semantic corruption that also satisfies $\mby \indep_{p_{tr}} \mbz \g T(\mbx, \mbdelta)$ the inequalities hold with $\epsilon=0$.
}
\begin{prop}\label{thm:estimation}
\propone{}
\end{prop}

If $\epsilon=0$, $p_T(\mby,\mbx)=\pind(\mby,\mbx)$ which means that almost surely the conditionals match $\pind(\mby\g \mbx) = p_T(\mby\g \mbx)$.
Then, as $\pind(\mby\g \mbx)$ is the optimal risk-invariant predictor, so is $p_T(\mby\g \mbx)$.
More generally, standard domain adaptation risk bounds that are controlled by the total variation distance  between source and target \citep[Theorem 1]{ben2010theory} bound the risk of a model under $\pind$ using the $L_1$ bound $m\epsilon$ --- which upper bounds the total variation --- 
and the risk under $p_T$.

Without nuisance annotations, one cannot test whether estimate the $L_2$-distance between the two biased models $\ptr(\mby\g \mbz)$ and $\ptr(\mby\g T(\mbx, \mbdelta))$ in \cref{thm:estimation}.
This distance can be large when a transformation $T(\mbx,\mbdelta)$ retains semantic information.
To avoid, we turn to a complementary source of knowledge: semantic features.
Using this knowledge, we design families of data augmentations that corrupt the semantic information in $\mbx$ to construct {semantic corruptions}.
Focusing on two popular \gls{ood} tasks, object recognition and \gls{nli}, we use \textbf{only semantic knowledge} to build corruptions that retain some aspects of the covariates.
Biased models built on such corruptions will depend on any retained nuisances; more retained nuisances mean better biased models.

\subsection{Semantic corruptions via permutations}

We first build corruptions when semantics appear as global structure.
We give an intuitive example for such {global} semantics.
Consider the waterbirds dataset from \citet{sagawa2019distributionally} with waterbirds and landbirds appearing predominantly on backgrounds with water and land respectively.
Semantic features like the wing shape and the presence of webbed feet are corrupted by randomly permuting small patches.
See \cref{fig:birds-pr}.
Formally, given subsets of the covariates $\mbx_1, \cdots \mbx_k$ extracted in an order, global semantics $e(\mbx_1, \cdots, \mbx_k)$ change with the order of extraction.
Formally, with a random permutation $\pi\sim q(\pi)$ and recalling that semantics are $\mbx^*=e(\mbx)$, the information about semantics is lost after permutation: $\forall \pd, \mbI_{\pd,q(\pi)}(\mbx^* ; e(\mbx_{{\pi(1)}}, \cdots \mbx_{{\pi(k)}}))) =0$.

We give an example of a semantic corruption with global semantics.
Consider distributions $\{\pd\}_{D\in \mathbf{R}}$  with different nuisance-label relationships.
With $\cU$ as the uniform distribution over $\{1,2,3\}$, and $\cN$ as the normal distribution, $\pd(\mby, \mbz, \mbx)$ corresponds to $\mby \sim \cU$, $\mbz \sim \cN(D\mby,1),$ and $\mby$ selecting a configuration of $\mbx$
\begin{align*}
\mby& =1 \implies \mbx = [ -\mbz, \mbz, \mbz],
\qquad 
    \mby =2 \implies \mbx = [ \mbz, -\mbz, \mbz],
\qquad
    \mby =3 \implies \mbx = [ \mbz, \mbz, -\mbz]
\end{align*}
The index of the negated coordinate is the semantic feature $\mbx^*$ that equals $\mby$ and computing it requires comparing coordinates: $\mby=1$ if $\mbx_2 \mbx_3 > 0$, $\mby=2$ if $\mbx_1 \mbx_3>0$, and $\mby=3$ otherwise.
In words, the semantic feature is global.
However, $\mbz = \mbx_1 + \mbx_2 + \mbx_3$ is determined regardless of where the negative sign is, i.e.\ it is not global.
A random permutation $T(\mbx, \mbdelta)$ of the coordinates in $\mbx$ is thus a semantic corruption: as $T(\mbx, \mbdelta)$ permutes the location of the negation, $T(\mbx,\mbdelta) \g \mby, \mbz$ is distributed identically to $T(\mbx,\mbdelta) \g \mbz$.
In turn, $T(\mbx, \mbdelta) \indep \mby\g \mbz$.
Further, the product of the three coordinates of $T(\mbx,\mbdelta)$ determines $\mbz$: $(\Pi_{i\in \{1,2,3\}}T(\mbx,\mbdelta)_i)^{\nicefrac{1}{3}} = - \mbz.$
Thus, $T(\mbx,\mbdelta)$ determines $\mbz$ and $\mby \indep \mbz \g T(\mbx, \mbdelta)$. 
These two independencies imply that $\epsilon=0$ in \cref{thm:estimation}.
Then, biased models from $T(\mbx)$ are as good as ones from $\mbz$.
Next, we give corruptions for global semantics in vision and language tasks, that retain non-global features.

\paragraph{\Acrlong{pr}.}
Object recognition tasks where the object is a shape that can satisfy the global property.
For illustration, consider differentiating cows and penguins in natural images; here, shape is a global semantic feature that structures multiple patches.
Permuting patches via \textit{\glsreset{pr}\gls{pr}}, like in \cref{fig:birds-pr}, corrupts global semantics.

\paragraph{\Acrlong{nr}.}
Tasks like \glsreset{nli}\gls{nli} --- where the goal is determining if a premise sentence entails a hypothesis --- satisfy the global-semantics property.
Consider this example: the sentence "Bob speaks but Jon does not" contradicts "Jon speaks but Bob does not" but entails "Bob speaks".
The meaning is inferred from a global structure on the words and the order they appear in.
Here, randomizing the order of the words corrupts the semantics:
For example, one randomized order of the sentence "Jon speaks but Bob does not" is "Bob speaks but Jon does not"; the former entails "Jon speaks" but the latter contradicts it.
We randomize the order
by permuting different $n$-grams in each sentence; we call this \textit{\glsreset{nr}\gls{nr}}.
 
\begin{figure*}[t]
\vspace{-10pt}
\centering
\begin{subfigure}{.49\textwidth}
  \centering
  \includegraphics[width=0.9\linewidth]{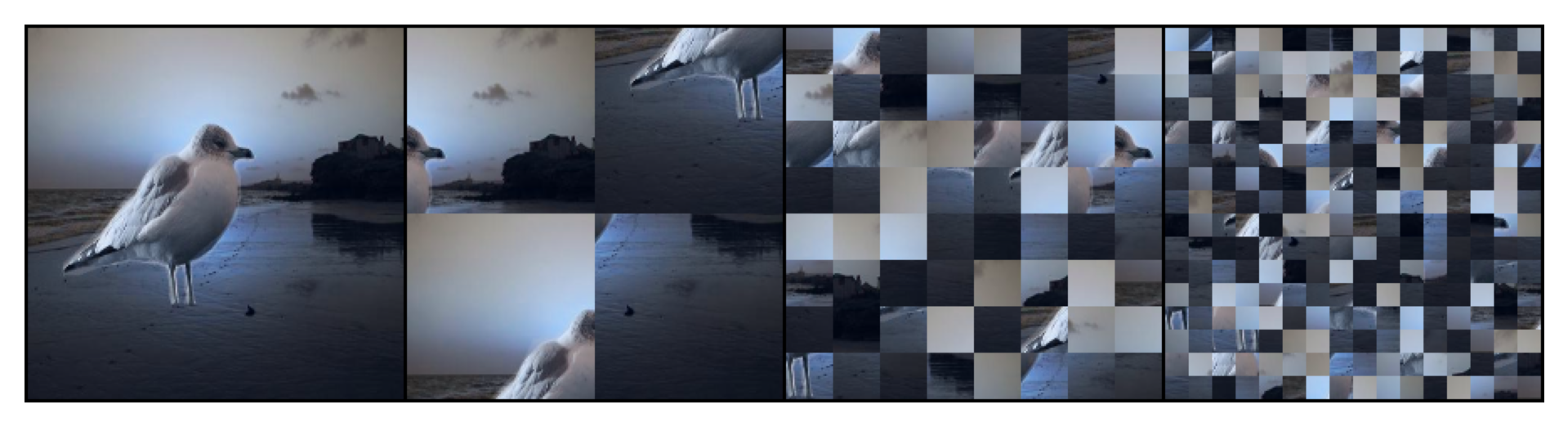}
  \caption{ \gls{pr} to corrupt global semantics in Waterbirds. The original is the left most, followed by \glspl{pr} with sizes $112,28,14$.
    At sizes $>28$, shape is hard to make out.}
  \label{fig:birds-pr}
\end{subfigure}%
\hspace{5pt}
\begin{subfigure}{.49\textwidth}
  \centering
  \includegraphics[width=0.9\linewidth]{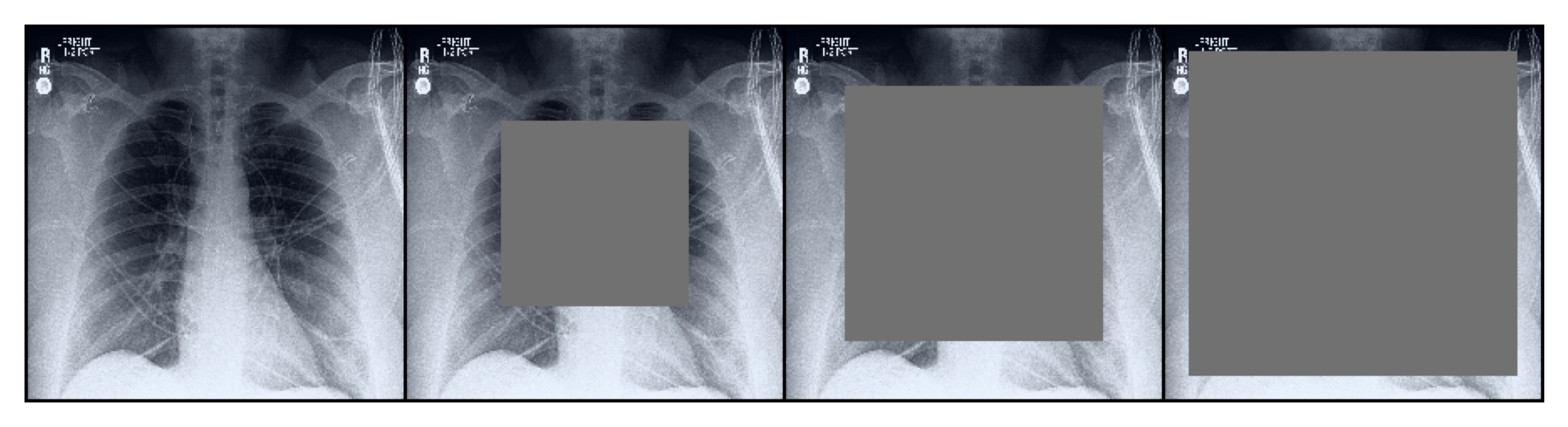}
  \caption{Masking to corrupt semantics in chest X-rays.
    The original is the left most, followed by  \acrshort{rm} of size $112,154,196$.
    At sizes $>154$, the heart is blocked out.}
  \label{fig:xray-masking}
\end{subfigure}
\vspace{-10pt}
\caption{Semantic corruptions of Waterbirds via \gls{pr} and chest X-rays via \acrshort{rm}.}
\label{fig:pr-and-mask}
\vspace{-5pt}
\end{figure*}
  
\subsection{Semantic corruptions via masking} \label{sec:masking}

The second corruption we build focuses on cases where certain subsets of the covariates are necessary part of semantics. 
Masking, by removing that subset or setting it to a constant, corrupts semantics.
Formally, we corrupt the semantics by replacing subsets $\mbx_S$ with a value that is out of support: for example, in images where pixels lie in $(0,1)$, we corrupt $\mbx =[\mbx_S, \mbx_{-S}]$ as $\mbx_{\text{corrupted}}=[0*\mbx_S, \mbx_{-S}]$.
As an illustrative example, consider a family $\cF=\{\pd\}_{D\in R}$ with varying nuisance-label relationships.
With $\mba,\mbb$ being uniform binary random variables, $\mbe(\rho)$ as the exponential distribution with parameter $\rho$, and $s_+(u) = \log(1 + \exp(u))$ as softplus, $\pd(\mby, \mbz, \mbx)$ describes:
\begin{align}
    \mby &= \mba \oplus \mbb, \,\qquad  \mbz\sim \mbe(s_+(D*(2\mby-1))), \qquad \mbx = [(2\mba-1)\mbz, (2\mbb-1)\mbz].
\end{align}
For such a family, we show that masking out coordinate $\mbx_1$ is a semantic corruption: $T(\mbx) = [0,\mbx_2]$ satisfies $T(\mbx) \indep \mby \g \mbz$ and $T(\mbx) \nindep \mbz$.
First, due to $\mby$ being computed as an XOR function of $\mba, \mbb$, it holds that $\mbb \indep \mby$.
Then, due to $\mbz$ only relying on $\mby$ and exogenous noise, $\mbb \indep (\mby, \mbz)$ which implies $\mbb\indep \mby \g \mbz$. 
Given $\mbz$, $\mbb$ determines $\mbx_2$, so $\mbb\indep \mby \g \mbz \implies \mbx_2 \indep \mby \g \mbz \implies T(\mbx) \indep \mby\g \mbz$.
Further, $\|T(\mbx)_2\| = \mbz$ which means $\mby \indep \mbz \g T(\mbx)$.
These two independencies imply that $\epsilon=0$ in \cref{thm:estimation}.
Then, using $T(\mbx)$ to build biased models is equivalent to building them with $\mbz$.

\paragraph{ROI-masking for object recognition.}
Semantics in images can often be localized to a \gls{roi}. 
For example, in detecting cardiomegaly, the \gls{roi} is the chest where the heart resides.
Masking out the \gls{roi} removes centrally located semantic information from the chest X-ray (\cref{fig:xray-masking}).
We call this \textit{\gls{rm}}.

\paragraph{Premise-masking for NLI.}

Semantic features in \gls{nli} rely on the meanings of the premise and the hypothesis sentences: for example, the premise states the occurrence of an event (``Alice sat while Bob stood.'') which can entail (``Alice sat.'') or contradict (``Bob sat.'') the hypothesis.
The information about the setup in the premise is therefore crucial to detect entailment or contradiction.
If the context given by the premise is blocked out, the hypothesis sentence can predict the label only due to nuisances.
Thus, masking the premise is a semantic corruption for \gls{nli} that retains hypothesis features; we call this \textit{\gls{pm}}.

{ 
\subsection{Semantic corruptions via frequency and intensity filters}

\Gls{pr} relies on differences in relative size and \gls{rm} relies on differences in spatial position.
We consider two aspects of the image that are not spatial, frequency and pixel-intensity, and give corruptions for features that depend on these aspects.
Semantics can appear as signals in a particular region of the frequency spectrum, or appear at a particular luminosity in the image.
For example, consider detecting cardiomegaly from chest X-rays, where the heart appears as an object formed of bright pixels with little variation in intensity across the pixels; the latter suggests that the heart features are low-frequency signals.

This observation motivates corruptions along the axes of frequency and pixel-intensity: \textit{\gls{ff}} and \textit{\gls{if}}.
\Gls{ff} zeroes-out frequencies in the discrete fourier domain, while {\gls{if}} zero-out pixels based on their intensities.
See \cref{fig:xray-alts} for how \gls{ff} and \gls{if} corrupt the heart region.
\gls{ff} and \gls{if} require characterizing semantic features in frequency and intensity spaces; this is in contrast to \gls{rm} that is based on characterizing the position in pixel space that the semantics occur in.
}

\begin{figure}
\centering
\begin{subfigure}{.47\textwidth}
  \centering
  \includegraphics[width=0.9\linewidth]{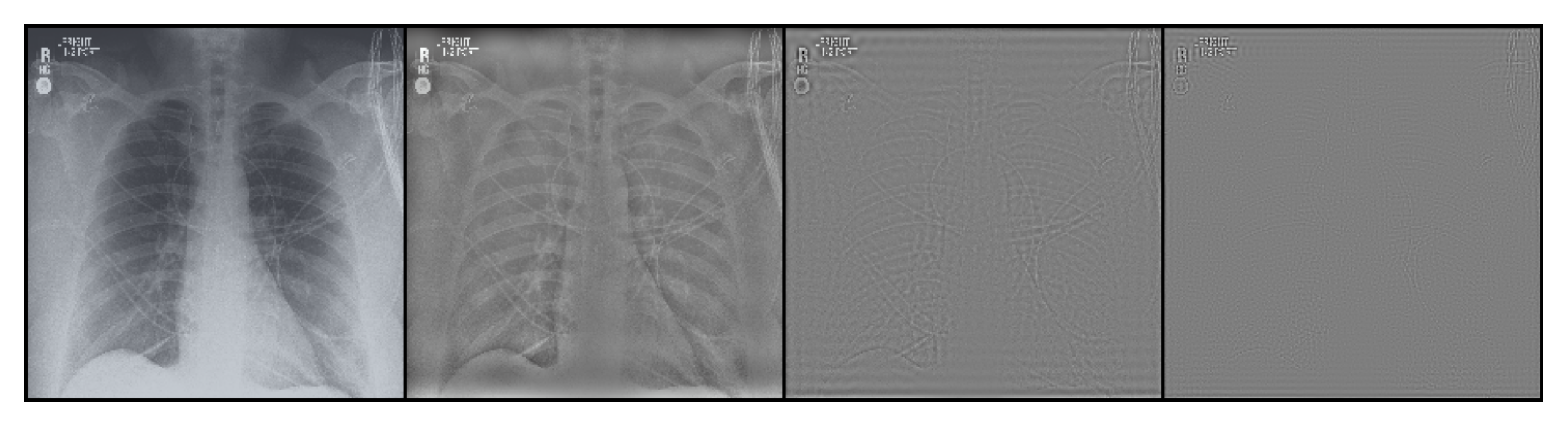}
  \caption{Corruption via \gls{ff}. Original image to the left followed zeroing out $14,56, 112$ of the lowest frequencies. The heart features are corrupted at $56$.
  }
  \label{fig:freq}
\end{subfigure}%
\hspace{10pt}
\begin{subfigure}{.47\textwidth}
  \centering
  \includegraphics[width=0.9\linewidth]{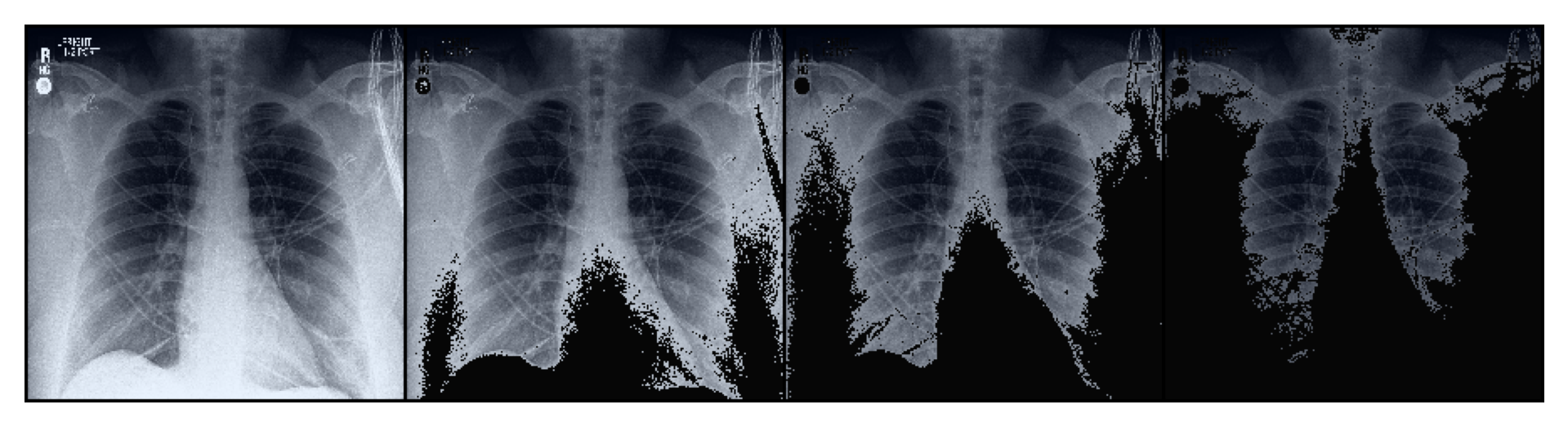}
  \caption{Corruption via \gls{if}. Original image to the left followed by zeroing out pixels with intensities above the $80\%, 60\%, $40\%. Heart features look corrupted at $40\%$.}
  \label{fig:intensity}
\end{subfigure}
\caption{Semantic corruptions of chest X-rays via \gls{ff} and \gls{if} respectively.}
\label{fig:xray-alts}
\vspace{-5pt}
\end{figure}

\subsection{Using semantic corruptions in practice}\label{sec:corruptions-in-practice}

For each method in \cref{tab:nmeth-summary}, we use a semantic corruption $T(\mbx)$ in building a model $\ptr(\mby \g T(\mbx))$.
For reweighting-\gls{nurd}, we replace $\ptr(\mby\g \mbz)$ with $\ptr(\mby \g T(\mbx))$, for \gls{dfl} and \gls{poe}, we replace the model $\ptr(\mby \g \mbz)$ with $\ptr(\mby\g T(\mbx))$, and for \gls{jtt}, we use $\ptr(\mby \g T(\mbx))$ as the identification model.

\textbf{Choosing the corruption parameter.}
To corrupt with \gls{pr}, \gls{nr}, and \gls{rm}, \gls{ff}, one must select a size parameter and to corrupt with \gls{if}, one must specify an intensity threshold.
For \gls{nurd}, \gls{jtt}, \gls{poe} and \gls{dfl}, we select corruption parameters with the same validation schemes used to select other hyperparameters in each respective paper.
In practice, including the \glspl{scam} run without semantic corruptions in the \gls{scam}'s validation scheme ensures a lower bound on performance.
For example, for \gls{jtt}, this inclusion yields a lower bound that corresponds to vanilla \gls{jtt}'s performance.
We also report results for all corruption parameters in \cref{appsec:all-results}, showing that all semantic corruptions except \gls{if} are not sensitive to the parameters, and lead to models that outperform \gls{erm}.

\section{Experiments}\label{sec:exps}

We study semantic corruptions in powering \gls{nurd} \citep{puli2021predictive}, \gls{jtt} \citep{liu2021just}, and \gls{poe} and \gls{dfl} \citep{mahabadi2019end}. 
To be faithful to the original evaluations of each method, we run them on tasks from their respective papers: \gls{nurd} on waterbirds, \gls{jtt} on waterbirds and \gls{nli} where the nuisance is the presence of a negation word, and \gls{poe} and \gls{dfl} on \gls{nli} evaluated on a challenging test dataset, HANS \citep{mccoy2019right}.
We run \gls{nurd} on chest X-rays but focus on detecting cardiomegaly rather 
than the original pneumonia \citep{puli2021predictive} because pneumonia detection even with known-nuisances is not performant.
See \cref{appsec:exps} for details and \cref{appsec:all-results} for additional experiments investigating semantic  corruptions.

\paragraph{Methods, metrics and model selection.}
For images, we corrupt semantics with \gls{pr}, a central \gls{rm}, \gls{ff}, and \gls{if}.
To show the value of semantic corruptions relative to existing data augmentations, we also consider two baseline transformations of images.
The first is random cropping (\textsc{rand-crop}) like in self-supervised learning \citep{bardes2021vicreg,chen2020simple} where patches of random sizes are sampled, covering $\geq 0.08$ fraction of the image.
 The second is adding gaussian noise (\textsc{gauss-noise}).
 For text, we corrupt semantics with \gls{nr} and \gls{pm}.
We report the average test accuracy for every method.
To be able to compare to what \gls{jtt} is trained for in \citet{liu2021just}, we report worst-group test accuracy for \gls{jtt}.
For each method, we compare the performance of the original method to that of the methods run with semantic corruption (including the baselines).
For the corruption-powered versions, group annotations and nuisances are \textit{unavailable} in the training data.
Known-nuisance versions of \gls{poe}, \gls{dfl}, and \gls{nurd} use direct knowledge of one or more nuisances during training.
In choosing parameters and early stopping, like \citet{liu2021just} do with vanilla \gls{jtt}, corruption-powered \gls{jtt} uses validation group annotations.
For the other methods,
we follow validation schemes from the respective papers: for \gls{nurd} we follow \citet{puli2021predictive} and use a validation set weighted to have independent nuisance and label, and for \gls{poe}/\gls{dfl}, we follow \citet{mahabadi2019end} and use a set of $1000$ samples from the HANS training dataset.

\subsection{Object recognition tasks}

\begin{wraptable}[24]{r}{0.34\textwidth}
\centering
\vspace{-14pt}
 \caption{Mean and standard error of test accuracy across $10$ seeds of \gls{nurd} with semantic corruptions on waterbirds.
 \textit{Known}-$\mbz$ \gls{nurd} uses a label for the type of background as the nuisance.
Consider the gap between \gls{erm} and known-nuisance \gls{nurd}.
\Gls{nurd} with semantic corruptions \gls{pr}, \gls{rm}, \gls{ff}, and \gls{if} close $99\%,99\%,82\%,99\%$ of the gap respectively.
\Gls{nurd} with semantic corruptions outperforms \gls{erm} and \gls{nurd} with \textsc{rand-crop}, \textsc{gauss-noise}.
 }
    \label{tab:nurd-wb-results}
     \centering
  \vspace{-16pt}
\begin{tabular}{lc}
\\
    \toprule
      Method &  test acc. \\
    \midrule
        \textit{Known}-$\mbz$ \gls{nurd}
        &  $87.2 \pm 1.0 \%$ 
\\
	\midrule
\gls{pr}          
        &  $  86.9 \pm 1.2 \% $
    \\
\gls{rm}          
        &  $  86.9 \pm 1.7 \%$
    \\
\gls{ff}          
        &  $  83.5 \pm 1.1  \%$
    \\
\gls{if}          
        &  $  86.9 \pm 1.1  \%$
    \\
	\midrule   
\textsc{rand-crop}       
        &  $  73.7 \pm 2.0  \%$
    \\
\textsc{gauss-noise}        
        &  $  82.0 \pm 2.6  \%$
    \\
	\midrule    
            \gls{erm} 
            & $68.0 \pm 1.9\%$
	\\
	\bottomrule
\end{tabular}
\end{wraptable} 
To be faithful to the original evaluations of each method, we evaluate \gls{jtt} on waterbirds, and \gls{nurd} on both waterbirds and detecting cardiomegaly; both tasks have images of size $224\times 224\times 3$.
    Both \citet{puli2021predictive} and \citet{liu2021just} use the raw waterbirds data from \citet{sagawa2019distributionally},
    where the task is detecting the type of bird (water or land) from images where the background is a nuisance.
    Unlike \citet{liu2021just}, \citet{puli2021predictive} process the waterbirds to get a different setup from \citet{sagawa2019distributionally}.
    To stay true to the original evaluations of the methods, we recreate the setups as described in their respective papers.
For both tasks, we use \gls{pr} (of patch sizes $7,14,28,56$), \gls{rm} (of mask sizes $112,140,168,196$), \gls{ff} (of high-pass filter sizes $196,168,140, 112$), and \gls{if} (of thresholds $0.1,0.2,0.3,0.4$) as semantic corruptions.
For \textsc{gauss-noise}, we use variances $0.01, 0.25, 1, 4$.

\begin{wraptable}[17]{r}{0.35\textwidth}
\centering
\vspace{-13pt}
    \caption{
    Test worst-group (WG) accuracies of \gls{jtt} on waterbirds.
\gls{jtt} with semantic corruptions outperforms \gls{erm}, vanilla \gls{jtt}, and \gls{jtt} with baseline corruptions (\textsc{rand-crop}, \textsc{gauss-noise}).
}
    \label{tab:jtt-wb-results}
  \centering
  \vspace{-16pt}
\begin{tabular}{lc}
\\
    \toprule
      Method &  test WG acc. \\
    \midrule
        \textit{Vanilla} \gls{jtt} 
        &  $86.5 \%$ 
\\
	\midrule
\gls{pr}          
        &  $ 89.0\%$ 
    \\
\gls{rm}          
        &  $ 88.2\%$ 
    \\
\gls{ff}          
        &  $ 87.2\%$ 
    \\
\gls{if}          
        &  $ 87.0\%$ 
    \\
	\midrule   
\textsc{rand-crop}       
        &  $ 75.0\%$ 
    \\
\textsc{gauss-noise}        
        &  $ 71.0\%$ 
    \\
	\midrule    
            \gls{erm} 
            & $72.0\%$
	\\
	\bottomrule
\end{tabular}
\end{wraptable}

\paragraph{\Gls{nurd} on waterbirds.}

For \gls{nurd}, we recreate the waterbirds experiment from \citet{puli2021predictive} where the full waterbirds data from \citet{sagawa2019distributionally} was subsampled into training, validation, and test datasets each with label balance.
However, unlike \citet{sagawa2019distributionally}, the validation data comes from the same distribution as the training data.
The training and validation datasets have $90\%$ waterbirds on backgrounds with water and $90\%$ landbirds on backgrounds with land.
The test data has a flipped relationship.
Known-nuisance \gls{nurd} uses an additional label denoting the background-type as the nuisance.

\Cref{tab:nurd-wb-results} gives results. Selecting hyperparameters using \gls{nurd}'s validation approach gives sizes $14$ for \gls{pr} ($86.9\%$), $196$ for \gls{rm} ($86.9\%$), $168$ for \gls{ff} ($83.5\%$), and threshold $0.2$ for \gls{if} ($86.9\%$).
Consider the gap between \gls{erm} and known-nuisance \gls{nurd}.
\gls{nurd} with \gls{pr}, \gls{rm}, \gls{ff}, and \gls{if} close $99\%,99\%,82\%,99\%$ of the gap respectively.
\gls{nurd} with these semantic corruptions outperforms \gls{erm} ($68.0\%$) and \gls{nurd} with \textsc{rand-crop} ($73.7\%$) and \textsc{gauss-noise} ($82.0\%$).
Additionally, in \cref{tab:nurd-wb-results-full} in \cref{appsec:exps}, we give the results for all corruption parameters, showing that \gls{nurd} with semantic corruptions is \emph{insensitive to hyperparameters of the corruption} and outperforms \gls{erm}.
In \cref{appsec:remark-on-baselines}, we discuss how the baseline \textsc{gauss-noise} could close $80\%$ of the gap between \gls{erm} and known-$\mbz$ \gls{nurd}.

\paragraph{JTT on waterbirds.}

For \gls{jtt}, we repeat the waterbirds experiment from \citet{liu2021just}
which uses the original data from \citet{sagawa2019distributionally}.
The training data has $95\%$ waterbirds on backgrounds with water and $95\%$ landbirds on backgrounds with land.
Both the validation and test datasets have bird label independent of the background.
The groups here are subsets of the data that correspond to a pair of values of bird-type and background-type.
Like vanilla \gls{jtt}, we use group annotations in the validation data to compute worst-group error and early stop training when using \gls{pr} and \gls{rm}.
The results for vanilla \gls{jtt} are from our run using the optimal hyperparameters from \citet{liu2021just}.

\begin{wraptable}[22]{r}{0.35\textwidth}
\centering
\vspace{-14pt}
    \caption{
  Mean and standard error of test accuracy over $10$ seeds of \gls{nurd} on chest X-rays.
   \textit{Known}-$\mbz$ \gls{nurd} uses the hospital as the nuisance.
Consider the gap between \gls{erm} and known-$\mbz$ \gls{nurd}.
\gls{nurd} with \gls{pr}, \gls{rm}, \gls{ff}, and \gls{if} close $72\%,82\%,65\%,35\%$ of the gap respectively.
Except with \gls{if}, \gls{nurd} with semantic corruptions outperforms \gls{erm} and \gls{nurd} with baseline corruptions.
  }
    \label{tab:xray-results}
     \centering
\vspace{-17pt}
\begin{tabular}{lc}
\\
    \toprule
      Method & test acc. \\
    \midrule
        \textit{Known}-$\mbz$ \gls{nurd}
        &  $81.7 \pm0.3 \%$ 
\\
	\midrule
\gls{pr}   
            &   $ 77.0 \pm 1.2 \%$ 
    \\
\gls{rm}          
            &   $ 78.7 \pm 0.3 \%$ 
    \\
\gls{ff}  
            &  $ 76.0 \pm 0.6 \%$ 
    \\
\gls{if}          
            &  $ 71.0 \pm 1.0 \%$ 
    \\
	\midrule   
\textsc{rand-crop}       
        &  $ 59.9 \pm 2.1 \%$
    \\
\textsc{gauss-noise}        
        &  $ 69.0 \pm 1.9 \%$ 
    \\
	\midrule    
            \gls{erm}  
            & $65.3 \pm 1.1 \%$
	\\
	\bottomrule
\end{tabular}
\end{wraptable}

\Cref{tab:jtt-wb-results} shows the results.
Selecting the corruption hyperparameters on the validation worst-group accuracy gives size $14$ for \gls{pr} ($89\%$), size $196$ for \gls{rm} ($88.2\%$), size $112$ for \gls{ff} ($87.2\%$), and threshold $0.4$ for \gls{if} ($87.0\%$).
\Gls{jtt} with these semantic corruptions outperforms \gls{erm} $(72.0\%)$, vanilla \gls{jtt} ($86.5\%$), and \gls{jtt} with the baseline corruptions \textsc{rand-crop} ($75\%$) and \textsc{gauss-noise} ($71\%$).
Additionally, \cref{tab:appsec-jtt-wb-results} shows that \gls{jtt} with \gls{pr} and \gls{rm} outperforms \gls{jtt} with the baseline corruptions and \gls{erm} at every patch/border-size.

\paragraph{\Gls{nurd} on detecting cardiomegaly}

In chest X-ray classification, differences between hospitals, like the scanners used to produce the X-rays, are known to correlate thoracic conditions with non-physiological aspects in the image; for example, only some scanners render the air in the lungs in white \citep{zech2018variable}.
We consider the shape-based object recognition task of cardiomegaly (an irregularly sized heart) detection and, following \citet{puli2021predictive}, construct a dataset from two chest X-ray datasets: chexpert \citep{irvin2019chexpert} and MIMIC \citep{johnson2019mimic}.
The training and validation datasets have $90\%$ cardiomegaly images from MIMIC and $90\%$ healthy images from chexpert, while the test data has a flipped relationship.
Known-nuisance \gls{nurd} uses hospital identity as the nuisance.

See~\cref{tab:xray-results} for results.
Selecting the corruption parameters using \gls{nurd}'s validation approach gives size $14$ for \gls{pr} ($77\%$), size $196$ for \gls{rm} ($78.7\%$), size $168$ for \gls{ff} ($76.0\%$), and threshold $0.1$ for the \gls{if} ($71.0\%$).
 Consider the gap between \gls{erm} and known-nuisance \gls{nurd}.
\gls{nurd} with \gls{pr}, \gls{rm}, \gls{ff}, and \gls{if} close $72\%,82\%,65\%,35\%$ of the gap respectively.
\gls{nurd} with all semantic corruptions, outperforms \gls{erm} ($65.3\%$) and \gls{nurd} with  the baselines \textsc{gauss-noise} ($69\%$) and \textsc{rand-crop} ($59.9\%$).
Additionally, we report results for all corruptions in \cref{tab:nurd-wb-results-full} in \cref{appsec:exps} showing that \gls{nurd} with \gls{pr} and \gls{rm} \emph{are insensitive to  hyperparameters} and outperform \gls{erm}.

\subsection{\glsreset{nli}\Gls{nli}}

\begin{wraptable}[20]{r}{0.34\textwidth}
\centering
\begin{small}	
\vspace{-13pt}
 \caption{Mean and standard deviation of accuracies (over $4$ seeds) on the HANS dataset.
The results for \gls{poe} and \gls{dfl} that use known nuisances are given under \textit{known}.
        \gls{poe} with  \gls{nr} (\textsc{nr}) performs better than known-nuisance \gls{poe}. \gls{dfl} with (\textsc{nr}) closes ${84}\%$ of the gap between \gls{erm} and known-nuisance \gls{dfl}.
\Gls{poe} and \gls{dfl} with \gls{pm} (\textsc{pm}) close $33\%$ and $28\%$ of the gap between \gls{erm} and the respective method with known $\mbz$.
  }
    \label{tab:nli-results}
 \vspace{-15pt}
  \centering
    \begin{tabular}{lc}
    \\
    \toprule
Method  & 
HANS test acc. \\
    \midrule
        \gls{poe},      \textit{known}-$\mbz$  
        &  $66.3 \pm 0.6 \%$ 
\\
        \gls{poe}, \textsc{nr}
        & $ 66.7  \pm 1.5 \%$ 
\\
        \gls{poe}, \textsc{pm}
        & $ 64.5 \pm 1.9 \%$ \\
\midrule
        \gls{dfl}, \textit{known}-$\mbz$
        & $69.3 \pm 0.2\%$

\\
 \gls{dfl}, \textsc{nr}
        & $68.4  \pm 1.5\%$ 
\\
 \gls{dfl}, \textsc{pm}
        & $ 65.2 \pm 0.7\%$ 
\\
\midrule
 \gls{erm}
        & $63.6  \pm 1.1\%$  \\
	\bottomrule
\end{tabular}
\end{small}
\vspace{-10pt}
\end{wraptable}

For methods \gls{poe}, \gls{dfl}, and \gls{jtt}, we use the MNLI dataset \citep{N18-1101} to fine-tune a BERT model.
The evaluations of these methods in their respective papers have different nuisances and, consequently, different test sets. 
In accordance, we describe the setup and the results separately.
We use \gls{nr} (sizes $1,2,3,4$) to produce nuisances for both setups.

\paragraph{PoE and DFL}
For \gls{poe} and \gls{dfl}, we report test accuracies on the HANS dataset \cite{mccoy2019right} as in \citet{mahabadi2019end}.
HANS was created to test the reliance of models on three known nuisances: 1) lexical overlap, 2) subsequence match, and 3) constituent matching subtrees in the parse trees.
Known-nuisance \gls{poe} and \gls{dfl} use exact knowledge of these nuisances.

\Cref{tab:nli-results} gives the mean test accuracies over $4$ seeds.
For both \gls{dfl} and \gls{poe}, selecting the size hyperparameter based on the average accuracy on a small subset of the HANS training data ($1000$ samples) gives $n=3$.
With this size, \gls{poe} achieves $66.7\%$, improving over \gls{poe} with known nuisances ($66.3\%$).
\gls{dfl} with \gls{nr} of size 3, achieves $68.4\%$, closing ${84}\%$ of the gap between \gls{erm} and known-nuisance \gls{dfl} ($69.3\%$).

\Gls{poe} and \gls{dfl} with \gls{pm} (\textsc{pm}) close $33\%$ and $28\%$ of the gap between \gls{erm} and when the methods have knowledge of $\mbz$.
We expect the methods with \gls{nr} do better than with \gls{pm} because the latter corrupts nuisances like lexical overlap between premise and hypothesis that HANS focuses on.
Additionally, we give results for all $n$-gram sizes in \cref{tab:nli-results-full} in \cref{appsec:exps}, showing that \gls{poe} and \gls{dfl} beat \gls{erm} for all $n$-gram sizes.
Further, in \cref{appsec:added-evaluations}, we evaluate \gls{poe} and \gls{dfl} models on the ANLI \citep{nie2019adversarial} dataset and counterfactually-augmented data \citep{kaushik2019learning} in \cref{tab:anli_results,tab:cad_results}.

\begin{wraptable}[11]{r}{0.35\textwidth}
\centering
\vspace{-14pt}
\caption{\small
Worst-group and average test accuracies of \gls{jtt} on \gls{nli}.
\gls{jtt} with \gls{pm} (\textsc{pm}) and \gls{nr} (\textsc{nr}) outperforms vanilla \gls{jtt} and \gls{erm}.
  }
    \label{tab:jtt-nli-results}
  \centering
\vspace{-14pt}
    \begin{tabular}{ccc}
    \\
    \toprule
	 &  Worst-group &   Avg. \\
    \midrule
        \textit{Vanilla} \gls{jtt} &  $71.3 \%$ & $79.1\%$   
	\\
       \gls{jtt} +  \textsc{pm}  &  ${72.1}\%$ & $79.9
        \%$   
    \\
\gls{jtt} +  \textsc{nr}
         &  ${74.3}\%$ & $79.7\%$   
    \\
	\midrule    
            \gls{erm} & $67.9\%$ &  $82.4\%$
	\\
	\bottomrule
\end{tabular}
\end{wraptable}

\paragraph{JTT}

For \gls{jtt}, we repeat the \gls{nli} experiment from \citet{liu2021just}, where the presence of a negation word in the hypothesis sentence is the nuisance.
The groups here are subsets of the data that correspond to a value of the label and whether or not there is a negation word in the hypothesis.
Vanilla \gls{jtt} uses group annotations in the validation data to tune the hyperparameters and early stop training.
For each $n$-gram size, we run \gls{jtt} with \gls{nr} for two values of the number of epochs of training for the identification model: $2,3$.
Following the hyperparameter selection procedure from \citet{liu2021just}, for each $n$-gram size, we give the results for the run with the higher validation worst-group accuracy.
\textit{Vanilla} \gls{jtt} is run with the optimization hyperparameters from \citep{liu2021just}.

\looseness=-1
\Cref{tab:jtt-nli-results} gives the results. 
Selecting the size hyperparameter for \gls{nr} using validation worst-group accuracy, like \citet{liu2021just} do for \gls{jtt}, gives $n=1$ with test worst-group accuracy of ${74.3\%}$, better than vanilla \gls{jtt}'s $71.3\%$.
Additionally, \cref{tab:appsec-jtt-nli-results} shows that \gls{jtt} using \gls{nr} at \textit{every} size or \gls{pm} performs better than both vanilla \gls{jtt} ($71.3\%$) and \gls{erm} ($67.9\%$).

\section{Related work}\label{sec:related}

\glsreset{scam}\Glspl{scam} like
\citep{veitch2021counterfactual,clark2019don,puli2021predictive,he2019unlearn,makar2021causally} assume the nuisance is available as additional knowledge during training.
Semantic corruptions offer a complementary approach to hand-crafting nuisances or obtaining auxiliary labels, by capturing nuisances that remain after corruption (e.g. non-global nuisances remain after \gls{pr}).
\Nmeths{} like LFF \citep{nam2020learning}, UMIX \citep{han2022umix}, and \gls{jtt} \citep{liu2021just} all rely on one crucial assumption: that \gls{erm}-training builds a biased model that exploits the nuisance and use it to reduce a second model's dependence on the nuisance.
Each method trains the second model with weighted cross-entropy with higher weights for samples misclassified by the biased model; the methods differ in how they build the biased model and how they compute the weighted loss.
The biased models learn to predict the label from the covariates.
Such a model can also rely on the semantic features and upweighting its misclassified samples produces data with a different label-semantic relationship from the one in the training data.
Models trained on such data are suboptimal on test data which has the same semantic relationship as the training data.
Using semantic corruptions in these \nmeths{} reduces the biased model's reliance on the semantics and makes the second model rely more on the semantics;
thus, \glspl{scam} that rely on assumptions on \gls{erm}-trained models being biased achieve better performance when using semantic corruptions. 
The experiments in \cref{sec:exps} confirm this empirically: \gls{jtt} with semantic corruptions improves over vanilla \gls{jtt}.

Two instances of semantic corruptions, \gls{pm} and \gls{rm}, appear in earlier work \citep{mahabadi2019end,he2019unlearn,puli2021predictive} but were designed using knowledge of where nuisances appear in the covariates.
\citep{puli2021predictive} used the borders of X-ray images as features that are related only to the scanner type (the nuisance), and not human physiology, to avoid spurious correlations in the detection of cardiomegaly.
For \gls{nli}, \citet{mahabadi2019end} use knowledge that the test set was constructed from samples misclassified by a model that relies on the hypothesis alone to build a biased model using the hypothesis sentence.
These are special cases of \gls{rm} and \gls{pm} from \cref{sec:masking} repsectively.
Our work widely generalizes the observations from the papers above by formally defining and further realizing the abstraction of semantic corruptions in several instances and across applications.

\citet{bahng2020learning} uses \textsc{cnn}s with small receptive fields (RFs), to capture non-global nuisances.
However, their method is typically limited to very small filters because at size 3x3, deep neural networks like \textsc{vgg} detect global semantics like shapes.
In contrast, the size choice in \gls{pr} has no bearing on the choice of the model; we used default vision models.
\citet{LeBras2020AdversarialFO} automatically identify and remove examples with nuisances using adversarial filtering, but risk removing genuinely easy examples.
\citet{qin2021understanding} work solely with vision transformers and point out that nuisances are the only reason labels can be predicted from transformations akin to patch-randomized images.
They propose to encourage transformers to have predictions and representations of the original images be dissimilar from those of patch-randomized ones.
In contrast, our work applies to general flexible models and shows that semantic corruptions can be used to break the label's relationship with nuisances in the original images. 

\cite{yao2022improving,gao2023out} use  additional knowledge about nuisances or environments to corrupt nuisances in the covariates,
\cite{yao2022improving} corrupt nuisances in the covariates via the Mixup \citep{zhang2017mixup} of samples from different domains that share a label.
\cite{gao2023out} directly randomize nuisances; for example, in detecting animals in their natural habitats, they place segmented animal foregrounds onto random habitat backgrounds.
Unlike these methods, we design semantic corruptions using the complementary knowledge about semantics, which can be available even without knowledge about nuisances.
\citet{clark2019don,li2019repair} construct nuisances in the training stage using prior knowledge: for example, \citep{clark2019don} uses the first token of the hypothesis as a nuisance for a synthetic \gls{nli} task which was created to have the first token be spuriously correlated with the label. 
Another example is the VQA task where the question-type is used as the nuisance. The constructed nuisances are then used to build biased (or bias-only) models, or construct per-sample weights to de-bias the loss.
In contrast, we use knowledge about semantics to corrupt them; for example, the order of the words is a semantic feature that is corrupted by randomizing the order.
This construction does not use knowledge of the nuisance.

\citet{sinha2021negative} use techniques like \gls{pr} to restrict supports in self-supervised learning and generative modeling.
\citet{carlucci2019domain} use \gls{pr} images to encourage a model to recover semantic structure.
In contrast, we use \gls{pr} to corrupt semantics and build biased models that rely on the nuisances, which help build predictive models that avoid reliance on nuisances.
We focus on corrupting semantic features using simple procedures (like permuting, masking, filtering) while papers \citep{kaushik2019learning,teney2020learning,feder2022causal,kaushik2020explaining,eisenstein-2022-informativeness,wang2021robustness,wang2020identifying} focus on perturbing semantic features while keeping other features the same. 
 These transformations produce examples of different labels, and are called counterfactuals. These examples are used to counterfactually augment the training data \citep{kaushik2019learning}.
%
Constructing counterfactuals can be hard. Works like \citep{kaushik2019learning,teney2020learning,feder2022causal,kaushik2020explaining} rely on humans to create counterfactuals because it is difficult to automate semantic perturbation without changing nuisances. For example, consider classifying dogs versus cats. Creating a dog that looks like a specific cat is much harder than removing the cat from the image by masking out those pixels.

Methods like \citep{wang2021robustness,wang2020identifying} construct counterfactuals automatically, but require additional knowledge of how nuisances appear in the text. For example, \citet{wang2021robustness} matches sentences that have opposite labels while sharing most words. The non-shared words would then be considered semantic.
Techniques like the matching one above from \cite{wang2020identifying} are unrealistic beyond the task of sentiment classification. For example, consider the label of entailment or contradiction in NLI. Data samples with entailment as the label that contain negation words are rare. This makes it hard to find a good counterfactual for data samples labeled with contradiction. Further, matching is difficult in other modalities, like images, where covariates are continuous or high-dimensional and live in spaces where metrics are unclear.

\section{Discussion} \label{sec:discussion} 
We study the use of semantic knowledge in models robust to spurious correlations.
In \cref{thm:assumptions}, we show that additional knowledge is necessary to achieve \gls{ood} generalization even when the training and test distributions are coupled in a nuisance-varying family.
Then, \cref{thm:estimation} shows that a biased model built from a transformation of the covariates $T(\mbx, \mbdelta)$ --- that is $\ptr(\mby \g T(\mbx, \mbdelta)$ ---
can power \nmeths{} to avoid nuisances if the biased model $\ptr(\mby \g T(\mbx, \mbdelta))$ is close to $\ptr(\mby\g \mbz)$ in $L_2$ distance. 
There are two scenarios where this distance is large: the transformation does not corrupt semantics and it corrupts nuisances.
We use knowledge of the semantics to design semantic corruptions to avoid the first scenario.
\textit{Since we work without nuisances}, to avoid the second scenario --- that is to choose $T(\mbx, \mbdelta)$ that retain nuisances --- we use standard validation schemes in \nmeths{}.
Using semantic corruptions, practitioners can run different kinds of  \glspl{scam} (\gls{nurd}, \gls{jtt}, \gls{dfl}, \gls{poe}).
Corruption-powered methods like \gls{nurd} and \gls{dfl} perform close to how they would with known nuisances. 
For methods like \gls{jtt}, the corruption-powered versions perform better than their vanilla versions that rely on \gls{erm} on the raw covariates to yield nuisances.

\paragraph{Limitations.}

The quality of any semantic corruption, and thus the quality of the results, depends on the extent to which semantics are destroyed and nuisances are retained.
\Gls{pr} and \gls{nr} are built to corrupt global semantics, and therefore are most suitable for when the nuisances are local.
\Gls{rm} corrupts semantics in the \gls{roi} and \gls{pm} corrupts the semantic context in the premise; these are most suitable for when nuisances lie outside the \glsreset{roi}\gls{roi} or in the hypothesis respectively.
Finally, \gls{ff} and \gls{if} corrupt semantics in particular parts of the frequency and intensity spectrum, and are most suitable for when the nuisances and semantics lie in separate parts of the spectra.
Knowledge about the kind of nuisances present in a dataset can lead to better choices of semantic corruptions. 
Alternatively, one could use standard validation schemes to select a corruption, like we do in \cref{sec:exps}.

When applied blindly, the procedures we describe may retain semantics or corrupt nuisances.
\Gls{pr} and \gls{nr} may corrupt global nuisances and retain local semantics, 
\gls{rm} and \gls{pm} may corrupt nuisances that occur in the same region as the semantics, 
and \gls{ff} and \gls{if} may corrupt both semantics and nuisances if they appear at similar frequencies or intensity.
For example, when \gls{pr} is used blindly on covariates with non-global semantics, the biased model may rely on said semantics; this in turn guides the predictive model to ignore these semantics and, thus, lose predictive performance.
Alternatively, when nuisances are global, \gls{pr} may corrupt them.
For example in detecting cows and penguins, other nuisance animals (like dogs) may co-occur with cows more often; \gls{pr} would corrupt this nuisance animal.
Using \gls{pr} in a \nmeth{} for such tasks could lead to non-robust predictive models that rely on corrupted nuisances.

Our experiments suggest that it might be possible to guard against performance degradation due to blind usage of semantic corruptions if the corruption parameter is made a hyperparameter and selected using standard validation schemes. In both classifying waterbirds and \gls{nli}, there exist non-global semantics, like small beaks and individual words, that are not corrupted by \gls{pr} and \gls{nr} respectively. However, in our Waterbirds and \gls{nli} experiments, we show models built using semantic corruptions, with validated size choices, close more than $80\%$ of the gap in test performance between \gls{erm} and the methods that use known nuisances. 
Now, imagine the extreme case of running \gls{nurd}, \gls{poe}, \gls{dfl} with a semantic corruption that destroys all information in the covariates. 
Biased models would predict like random chance, and the resulting predictive models would be no less robust than \gls{erm}. On the other hand, methods like \gls{jtt} perform at least as well as their vanilla versions as long as the validation strategy used in vanilla \gls{jtt} covers the identity function as a corruption.
Future work could consider combining semantic corruptions as a way to better retain of nuisances.
Given the validation strategies for \glspl{scam}, a practitioner can easily validate over both single and hybrid corruptions.

\paragraph{Summary.} Semantic corruptions power \nmeths{} to build models robust to spurious correlations using knowledge about the semantic features.
Additional knowledge is always required to achieve such robustness, and existing work assumes access to nuisance annotations or that \gls{erm}-trained models rely on nuisances.
By developing semantic corruptions, we give an approach to use a new kind of additional knowledge, thereby enlarging the set of tasks where one can build robust models.
As discussed above, our experiments show that using semantic corruptions in \glspl{scam} leads to models more robust than \gls{erm} and \gls{jtt} even when the corruptions may have corrupted some nuisances.
These two properties demonstrate the value of semantic corruptions as a way to build robust models.

\section*{Acknowledgements}

The authors were supported by NIH/NHLBI Award R01HL148248, NSF Award 1922658 NRT-HDR: FUTURE Foundations, Translation, and Responsibility for Data Science, NSF CAREER Award 2145542, Grant ONR N00014-23-1-2634, Apple Scholars in AI/ML PhD fellowship, and Samsung Advanced Institute of Technology (Next Generation Deep Learning: From Pattern Recognition to AI). Nitish Joshi is supported by the NSF Graduate Research Fellowship grant number 1839302.

\bibliographystyle{unsrtnat}
\bibliography{ms}

\newpage
\appendix

\onecolumn

\section{Proofs and Discussion on Semantic Corruptions}\label{appsec:thm}
\setcounter{thm}{0}
\setcounter{prop}{0}

In this section we give the proofs of \Cref{thm:assumptions} and \Cref{thm:estimation}.
The first result shows that even if we know our training and test data are sampled from distributions in a nuisance varying family $\cF$, additional assumptions are required in order to learn a predictor that is robust across the entire family.

\begin{thm}
For any learning algorithm,
there exists a nuisance-varying family $\cF$ where predicting with $\pind(\mby=1 \g \mbx)$ achieves $90\%$ accuracy on all members 
such that given training data $\mby, \mbx$ from one member $\ptr\in \cF$, 
the algorithm cannot achieve better accuracy than predicting at random on some $\pte\in \cF$.
\end{thm}

\newcommand{\ponerho}{p_{1,\rho}}
\newcommand{\ptworho}{p_{2,\rho}}

\begin{proof}
At a high-level, we setup two nuisance-varying families $\cF_1=\{\ponerho\}, \cF_2 = \{\ptworho\}$ where 
\begin{enumerate}
\item 
There are members of each family that have the same distribution over $(\mby, \mbx)$.
We let this distribution over $\mby, \mbx$ be the training data.
\item 
	Thus looking at this training data alone, no algorithm can tell which family the test distribution will come from.
\item 
Then, the proof concludes by showing any predictor that performs better than the chance on all members of $\cF_1$, will perform worse than chance on a member of $\cF_2$.
\end{enumerate}

\paragraph{Defining the two families.}
We now define two nuisance-varying families $\cF_1 = \{p_{1,\rho}\}$ and $\cF_2 = \{\ptworho\}$.
For $a\in \{-1,1\}$, and $\alpha\in [0,1]$ let $\mbR_\alpha(a)$ be a probability distribution obtained by randomly flipping the sign of $a$ with probability $1-\alpha$:
\begin{align}
r \sim \mbR_\alpha(a) \implies \begin{cases}
p(r = a) = \alpha
\\
p(r = -a) = 1 - \alpha
\end{cases}
\end{align}

Then, define the family $\{p_{1,\rho}\}$ as the distributions resulting from the following sampling process: 
\begin{align*}
\mby & \sim \mbR_{0.5}(1)
\\
\mbz & \sim \mbR_\rho(\mby)
\\ 
\mbx^* & \sim \mbR_{0.9}(\mby)
\\
\mbx &= [\mbx^*, \mbz]
\end{align*}

The second family $p_{2,\rho}$ follows the same process except that the positions of the semantic feature and nuisance are flipped $\mbx=[\mbz,\mbx^*]$.
\textbf{Notice that predicting $\mby$ from $\mbx_1$ in $\cF_1$ and from $\mbx_2$ in $\cF_2$, achieves $90\%$ accuracy.}
In both families,  by construction, the following properties hold 
\[			p_{1,\rho}(\mby) = p_{2,\rho}(\mby)
 \quad \qquad
 		p_{1,\rho}(\mbz , \mby) = p_{2,\rho}(\mbz , \mby), 
 \quad\qquad 
		 p_{1,\rho}(\mbx^* , \mby) = p_{2,\rho}(\mbx^* , \mby), 
\quad \qquad 
		\mbx_1 \indep_{p_{\cdot,\rho}} \mbx_2 \g \mby . \]
If $\rho\not=0.9$, due to the flipping of the positions of $\mbx^*, \mbz$ between $p_{1,\rho}$ and $p_{2,\rho}$,
\[p_{1,\rho}(\mbx_1\g \mby) \not= p_{2,\rho}(\mbx_1 \g \mby) \qquad \qquad p_{1,\rho}(\mbx_2 \g \mby) \not=p_{2,\rho}(\mbx_2 \g \mby).\]
But when $\rho=0.9$, the distributions are the same: $p_{\cdot,\rho}(\mbx_1 \g \mby) \stackrel{\mathclap{\normalfont\mbox{d}}}{=} p_{\cdot,\rho}(\mbx_2 \g \mby) \implies p_{1,0.9}(\mby,\mbx) = p_{2,0.9}(\mby,\mbx).$
With this we let the training data come from $\ptr=p_{1,0.9}$.

\paragraph{Reducing accuracy computation to summing conditional probabilities.}
Now, we express the accuracy of any predictor $f(x_1, x_2) \in \{-1, 1\}$ of $\ponerho$:
\begin{align}
\text{ACC}_f(p_{1,\rho}) & = \E_{p_{1,\rho}(\mby, \mbx_1,\mbx)}\ind[\mby=f(\mbx_1,\mbx_2)] \nonumber
\\
& =  \sum_{x_1,x_2}p_{1,\rho}(\mby=f(x_1, x_2), \mbx_1=x_1, \mbx_2=x_2)  \nonumber
\\
& =  \sum_{x_1,x_2}p_{1,\rho}(\mbx_1=x_1, \mbx_2=x_2 \g \mby=f(x_1, x_2)) p_{1,\rho}(\mby=f(x_1, x_2))  \nonumber
\\
& = 0.5\sum_{x_1,x_2}p_{1,\rho}(\mbx_1=x_1, \mbx_2=x_2 \g \mby=f(x_1, x_2)) 
\label{eq:acc}
 \end{align}
With this expression, we have reduced computing the accuracy of a model $f(x_1,x_2)$ to taking one from a pair of numbers --- either $p_{1,\rho}(\mbx_1=x_1,\mbx_2=x_2\g \mby=1)$ or $ p_{1,\rho}(\mbx_1=x_1,\mbx_2=x_2\g \mby=-1)$  based on what $f(x_1,x_2)$ predicts --- for each possible value of $x_1,x_1\in\{-1,1\}^2$, summing them and multiplying by $0.5$.

\begin{table}[h]
\centering
\begin{tabular}{c|c|c|c|c|c|c|c}
$(x_1,x_2)$ & $(-1,-1)$ & $(-1, 1)$ & $(1, -1)$ & $(1,1)$ & $\text{ACC}_f(p_{1,0})$ acc & $\text{ACC}_f(p_{1,1})$  &  $\min$ \\
\toprule
0 &  1 &   1 &   1 &   1   & $0.50$ &  $0.50$ &  $0.50$ \\ 
1 &  1 &   1 &   1 &  -1   & $0.55$ &  $0.05$ &  $0.05$ \\ 
2 &  1 &   1 &  -1 &   1   & $0.05$ &  $0.55$ &  $0.05$ \\ 
3 &  1 &   1 &  -1 &  -1   & $0.10$ &  $0.10$ &  $0.10$ \\ 
4 &  1 &  -1 &   1 &   1   & $0.95$ &  $0.45$ &  $0.45$ \\ 
5 &  1 &  -1 &   1 &  -1   & $1.00$ &  $0.00$ &  $0.00$ \\ 
6 &  1 &  -1 &  -1 &   1   & $0.50$ &  $0.50$ &  $0.50$ \\ 
7 &  1 &  -1 &  -1 &  -1   & $0.55$ &  $0.05$ &  $0.05$ \\ 
8 & -1 &   1 &   1 &   1   & $0.45$ &  $0.95$ &  $0.45$ \\ 
9 & -1 &   1 &   1 &  -1   & $0.50$ &  $0.50$ &  $0.50$ \\ 
10 & -1 &   1 &  -1 &   1   & $0.00$ &  $1.00$ &  $0.00$ \\ 
11 & -1 &   1 &  -1 &  -1   & $0.05$ &  $0.55$ &  $0.05$ \\ 
$\Longrightarrow$ 12 & -1 &  -1 &   1 &   1   & $0.90$ &  $0.90$ &  $\boldsymbol{0.90}$ \\ 
13 & -1 &  -1 &   1 &  -1   & $0.95$ &  $0.45$ &  $0.45$ \\ 
14 & -1 &  -1 &  -1 &   1   & $0.45$ &  $0.95$ &  $0.45$ \\ 
15 & -1 &  -1 &  -1 &  -1   & $0.50$ &  $0.50$ &  $0.50$\\
\bottomrule
\end{tabular}	
\caption{The $16$ different functions that are possible when predicting a label in $\{-1,1\}$ from $\mbx \in \{-1,1\}^2$.
We compute the accuracies on $p_{1,0},p_{1,1}$ and report the minimum of the two.
The only predictor that achieves better than random chance accuracy (denoted by $\Longrightarrow$) is $f(x_1,x_2) = x_1$.}
\label{tab:enum}
\end{table}

\paragraph{Showing only a semantic predictor can achieve better accuracy than random chance on $\cF_1$.}

Next, we will show that the only way to achieve better accuracy than random chance on every member of $\cF_1$ is to predict with $f(x_1,x_2) = x_1$.
To show this, we will express the accuracy computation for two distributions $p_{1,0}$ and $p_{1,1}$ by constructing a table of values of $p_{1,\rho}(\mbx_1=x_1,\mbx_2=x_2\g \mby=1)$ and $ p_{1,\rho}(\mbx_1=x_1,\mbx_2=x_2\g \mby=-1)$ for $\rho=0$ and $\rho=1$ separately.

\begin{table}[H]
\centering
\begin{tabular}{ll|c|cc}
\multicolumn{2}{c}{}&\multicolumn{2}{c}{$p_{1,1}$}&\\
\multicolumn{2}{c}{}&\multicolumn{2}{c}{$\mbx_1$}&\\
\multicolumn{2}{c}{}& $-1$ & $+1$ &\\
\cline{3-4}
	&&\\
\multirow{2}{*}{$\mbx_2$}& $-1$ & 
$0, 0.9$ & $ 0, 0.1$\\
	&&\\
\cline{2-4}
	&&\\
& $+1$ & $ 0.1,  0$ & $0.9, 0$  & \\
\end{tabular}
\begin{tabular}{ll|c|cc}
\multicolumn{2}{c}{}&\multicolumn{2}{c}{$p_{1,0}$}&\\
\multicolumn{2}{c}{}&\multicolumn{2}{c}{$\mbx_1$}&\\
\multicolumn{2}{c}{}& $-1$ & $+1$ &\\
\cline{3-4}
	&&\\
\multirow{2}{*}{$\mbx_2$}& $-1$ & 
$ 0.1, 0$ & $ 0.9,  0$\\
	&&\\
\cline{2-4}
	&&\\
& $+1$ & $ 0,  0.9$ & $0, 0.1$  & \\
\end{tabular}
\end{table}

By definition of accuracy from \cref{eq:acc}, the accuracy of any predictor $f(x_1,x_2)$ comes down to picking one from the pair of numbers --- left one if prediction if $1$ and right otherwise --- from each element in the table, summing them and multiplying by $0.5$.
There are $16$ possible functions ($2$ possible predictions each for $4$ combinations of $x_1, x_2$) and we enumerate them in \cref{tab:enum}, showing that only $f^*(x_1,x_2) = x_1$ will perform better than chance on both distributions $p_{1,0}$ and $p_{1,1}$.

\paragraph{No predictor can achieve better accuracy than random on both $\cF_1$ and $\cF_2$.}
The earlier parts showed that the only predictor that achieves better accuracy than random chance on all of $\cF_1$ is one that only relies on $\mbx_1$, which equals the semantic feature $\mbx^*$ under $\ponerho$.
However, under $\ptworho$, $\mbx_1$ is the nuisance $\mbz$.
Then, the predictor $f^*(x_1,x_2)=x_1$ has zero accuracy under $p_{2,0}$ because under $p_{2,0}$, we have $\mbz\sim R_0(\mby)$  which means $\mbz\not=\mby$ with probability one:
\begin{align}
\text{ACC}_{f^*}(p_{2,0}) & = \sum_{x_1,x_2}p_{2,0}(\mby=f(x_1, x_2), \mbx_1=x_1, \mbx_2=x_2) = \sum_{x_1,x_2}p_{2,0}(\mby=x_1, \mbz=x_1, \mbx_2=x_2) 
 = 0 
\end{align}
\end{proof}

\subsection{Semantic corruptions, biased models, and proof of \cref{thm:estimation}}\label{appsec:propone}
We give the definition of a semantic corruption here and discuss how it implies alternative intuitive definitions before presenting the proof of \cref{thm:estimation} on using corruptions to build biased models.
\begin{thmdef}[Semantic Corruption]
\label{def:semantic_corruption}
A semantic corruption is a transformation of the covariates $T(\mbx, \mbdelta)$, where $\mbdelta$ is a random variable such that $\mbdelta \indep (\mby, \mbz,\mbx, \mbx^*)$, if
    \[\forall \, p_D\in \cF \quad T(\mbx, \mbdelta) \indep_{p_D} \mbx^* \g \mbz.\]
\end{thmdef}
Two other plausible definitions that come to mind are $T(\mbx, \mbdelta) \indep_\pind \mbx^*$ and that $\mby\indep_\pd T(\mbx, \mbdelta) \g \mbz$.
These are both intuitive properties that can be asked of a semantic corruption that is supposed to discards all information about semantics, provided that the $\mbz$ which we wish to retain holds no information on it (which is the case under $\pind$).  We now show that \cref{def:semantic_corruption} implies these two.

From the definition that if $T(\mbx, \mbdelta)$ is a semantic corruption, then it also holds that $T(\mbx, \mbdelta) \indep_{\pind} \mbx^*$: since $\mbx^*\indep_{\pind} \mbz$
\begin{align}
    \pind(T(\mbx, \mbdelta), \mbx^*) & = \E_{\pind(\mbz)}\pind(T(\mbx, \mbdelta), \mbx^* \mid \mbz) = \E_{\pind(\mbz)}\pind(T(\mbx, \mbdelta) \g \mbz) \pind(\mbx^* \mid \mbz) \\
     & = \pind(\mbx^*) \E_{\pind(\mbz)}\pind(T(\mbx, \mbdelta) \g \mbz)  = \pind(\mbx^*)\pind(T(\mbx,\mbdelta)).
\end{align}
A semantic corruption satisfies the second definition also because
\begin{align}
\label{eq:drop-tx}
\begin{split}
    p_D(\mby \vert T(\mbx), \mbz) &= \int p_D(\mby\vert \mbx^*, T(\mbx), \mbz)p_D(\mbx^* \vert \mbz, T(\mbx)) d\mbx^* = \int p_D(\mby\vert \mbx^*, \mbz)p_D(\mbx^* \vert \mbz, T(\mbx)) d\mbx^* 
    \\ 
    &= \int p_D(\mby\vert \mbx^*, \mbz)p_D(\mbx^* \vert \mbz)d\mbx^* = p_D(\mby \vert \mbz) 
\end{split}
\end{align}
First transition adds in integration over the values of $\mbx^*$, second one uses the property of the nuisance varying family that $\mbx \bot \! \! \bot_{p_D} \mby \vert \mbz, \mbx^*$ and therefore it is also conditionally independent for any $T(\mbx, \mbdelta)$. Then the third transition is due to $T(\mbx, \mbdelta)$ being a semantic corruption.
The next result shows that the more our semantic corruption captures information about the nuisance that is relevant to predicting $\mby$, the better we can approximate learning under $\pind$, which would yield the optimal risk-invariant predictor over $\cF$ \citep{makar2021causally}.

\subsubsection{Proof of \cref{thm:estimation}.}
Now, using the property in \cref{eq:drop-tx} that holds for semantic corruptions, we prove \cref{thm:estimation}.
\begin{prop}
\propone{}
\end{prop}

\begin{proof}
The $L_1$ distance between the distributions is bounded from above by a $\pind$-weighted $L_1$ distance between $\ptr(\mby\g\mbz)$ and $\ptr(\mby\g T(\mbx))$, upto a constant:
\begin{align}
\int_{y,x}& \left|\pind(\mby, \mbx) \, - \,  p_T(\mby, \mbx))	\right| dy dx 
\\
& = 
\int_{y,x}  \left|\int_{z} \ptr(\mby) \ptr(\mby, \mbz, \mbx)p(\mbdelta) \left[ \frac{1}{\ptr(\mby\g \mbz)} - \frac{1}{\ptr(\mby\g T(\mbx, \mbdelta))}\right] d z\right|d y d x
\\
& 
 = 
\int_{y,x} \left|\int_{z} \ptr(\mby) \ptr(\mby, \mbz, \mbx)p(\mbdelta) \left[ \frac{\ptr(\mby\g T(\mbx)) - \ptr(\mby\g \mbz)}{\ptr(\mby\g \mbz)\ptr(\mby\g T(\mbx,\mbdelta))} - \right] d z\right|d y d x
\\
& 
 = 
\int_{y,x}  \left| \E_{\ptr(\mbz)p(\mbdelta)}  \frac{\ptr(\mby)}{\ptr(\mby\g T(\mbx,\mbdelta))} p(\mbx\g \mby, \mbz) \left[  \ptr(\mby\g T(\mbx,\mbdelta)) - \ptr(\mby\g \mbz)\right]\right|d y d x
\\
& 
\leq 
\int_{y,x} \E_{\ptr(\mbz)p(\mbdelta)}\left|\frac{\ptr(\mby)}{\ptr(\mby\g T(\mbx,\mbdelta))}  p(\mbx\g \mby, \mbz) \left[ \ptr(\mby\g T(\mbx,\mbdelta)) - \ptr(\mby\g \mbz) \right]\right|d y d x
\\
& 
=
\int_{y,x, z} \ptr(\mbz)\ptr(\mby)p(\mbdelta)p(\mbx\g \mby, \mbz) \frac{1}{\ptr(\mby\g T(\mbx,\mbdelta))}    \left|\ptr(\mby\g T(\mbx,\mbdelta)) - \ptr(\mby\g \mbz)\right|d y d x dz \\
& 
=
\E_{\pind(\mby,\mbz,\mbx)p(\mbdelta)}\frac{1}{\ptr(\mby\g T(\mbx,\mbdelta))}    \left|\ptr(\mby\g T(\mbx,\mbdelta)) - \ptr(\mby\g \mbz)\right|
\\
& 
\leq 
\left(\sqrt{\E_{\pind(\mby, \mbx)p(\mbdelta)}\frac{1}{\ptr(\mby\g T(\mbx, \mbdelta))^2}}\right)\sqrt{\E_{\pind(\mby, \mbz, \mbx)p(\mbdelta)} \left|\ptr(\mby\g T(\mbx,\mbdelta)) - \ptr(\mby\g \mbz)\right|^2}
\end{align}
Substituting the bounds from the theorem statement completes the proof of the bound.

Finally, if $T$ is a semantic corruption, by \cref{eq:drop-tx}, it holds that 
\[\ptr(\mby\g T(\mbx, \mbdelta), \mbz) = \ptr(\mby\g \mbz).\]
Then, if it also holds that $\mby\indep_\ptr \mbz \g T(\mbx, \mbdelta)$, it holds that
\[\ptr(\mby\g T(\mbx, \mbdelta), \mbz) = \ptr(\mby\g T(\mbx, \mbdelta)).\]
Together this implies that almost everywhere in $\ptr(\mby, \mbz, \mbx)p(\mbdelta)$
\[\ptr(\mby\g T(\mbx, \mbdelta)) = \ptr(\mby\g \mbz) \implies \E_{\pind(\mby, \mbz, \mbx)p(\mbdelta)} \left|\ptr(\mby\g T(\mbx,\mbdelta)) - \ptr(\mby\g \mbz)\right|^2 = 0.\]
This shows that for a semantic corruption such that $\mby\indep_\ptr \mbz \g T(\mbx, \mbdelta)$, it holds that $\epsilon=0$.
\end{proof}

\section{Further details about \acrshortpl{scam} and related work}\label{appsec:background}

\paragraph{\Gls{nurd}.}

Focusing on mitigating spurious correlations,
\citet{puli2021predictive} identify a conditional that has performance guarantees on every test distribution within a family of distributions with varying nuisance-label relationships: $\pte \in \cF$.
They develop \gls{nurd} to learn the conditional using data only from $\ptr \not=\pte$.
\gls{nurd} uses 1) the \textit{\nrd}, $\pind(\mby, \mbz, \mbx) = p(\mby) \pind(\mbz)p(\mbx\g \mby, \mbz)$, where $\mbz\indep_\pind \mby$, and 2)  an \textit{uncorrelating representation} $r(\mbx)$ for which
$\mbz\indep_\pind \mby\g r(\mbx)$.
\gls{nurd} builds models of the form $\pind(\mby\g r(\mbx))$ using $r(\mbx)$ that are most informative of the label.

{  We run reweighting-\gls{nurd}, which uses a biased model $\ptr(\mby\g \mbz)$  as an importance weight to compute loss under the \nrd{}: $\pind(\mby,\mbz, \mbx) = \frac{\ptr(\mby)}{\ptr(\mby\g \mbz)}\ptr(\mby,\mbz,\mbx)$.}

To run reweighting-\gls{nurd} with semantic corruptions, we replace $\ptr(\mby\g \mbz)$ with $\ptr(\mby \g T(\mbx))$ for a semantic corruption $T(\mbx)$.
Semantic corruptions are noisy functions of $\mbx$: with noise $\mbeps$ such that $(\mby, \mbz, \mbx)\indep_\pd \mbeps$, $T(\mbx) = U(\mbx, \mbeps)$.
This implies
\begin{align*}
        \mby \indep_\pind \mbeps \g \mbx
    \implies 
        \mby  \indep_\pind \mbx, \mbeps \g \mbx 
     \implies 
            \mby \indep_\pind T(\mbx) \g \mbx 
\end{align*}

Thus, $r(\mbx) =\mbx$ is uncorrelating and $\pind(\mby\g \mbx)$ achieves the optimality guarantees in \citet{puli2021predictive}.
These optimality guarantees imply that regardless of the test nuisance-label relationship, $\pind(\mby\g \mbx)$ will achieve optimal performance within the class of models
like $\pind(\mby\g r(\mbx))$.

{ 
\paragraph{End-to-end bias mitigation.}
\citet{mahabadi2019end} consider two methods to train a \textit{biased} model $\ptr(\mby\g \mbz)$ and a base predictive model jointly to make the base model predict without relying on the biases.
The methods use and fine-tune a BERT model \citep{Devlin2019BERTPO} and do not propagate the gradients of the biased model to update the common parameters (token embeddings in this case).
They propose 1) \gls{poe}, where the \texttt{log} of the product of the predictions (the output probabilities) of the two models is used to compute the classification loss and 2) \gls{dfl}, where the biased model is used to weight the cross-entropy loss for the base model.

The intuition for \gls{poe} is that the samples for which the biased model classifies correctly will not contribute to the gradients of the base model; thus the base model focuses more on classifying samples that the biased model misclassifies.
The \gls{dfl} algorithm weights each sample as the biased model's predicted probability of all but the label, exponentiated with $\gamma > 0$. 
This downweights samples that the biased model classifies correctly which in turn mitigates the base model's reliance on a nuisance which only helps predict the downweighted samples correctly.

Formally, with a biased model $f_\theta(\mbz)$ and a predictive model $f_\gamma(\mbx)$ that output a vector of logits over classes, $\sigma$ denoting the soft-max function that maps logits to class-probabilities, and $\sigma(\cdot)_y$ denoting the softmax-probability of label $y$
\begin{align}
	\textsc{poe} & \quad \max_{\theta, \gamma}
		\sum_{i\in \texttt{training data}}
		\log \sigma(f_\theta(\mbz_i))_{y_i} + 
		\log \sigma(f_\gamma(\mbx_i))_{y_i}
\\
	\textsc{dfl} & \quad \max_{\theta, \gamma}
		\sum_{i\in \texttt{training data}} 
		\left(1 - \sigma(f_\theta(\mbz_i))_{y_i}\right)^\gamma 
		\log \sigma(f_\gamma(\mbx_i))_{y_i}
\end{align}
\citet{mahabadi2019end} build the biased model $f_\theta$ using known nuisances $\mbz$.
We build this model from a semantic corruption $T(\mbx)$.

\paragraph{Just Train Twice (JTT).}
\gls{jtt} works in two stages: 1) build an "identification" model via \gls{erm} on the training data to isolate samples that are misclassified due to reliance on the nuisances and 2) train a model via \gls{erm} on data with the loss for the misclassified samples upweighted (by constant $\lambda$).
{ 
The identification model in \gls{jtt} is built to be a biased model.
When the identification model equals $\ptr(\mby\g \mbz)$, it exactly misclassifies the samples in the groups in the minority group\footnote{The minority group is the set of samples that the nuisance misclassifies. For example, when $\ptr(\mby=\mbz) > \ptr(\mby\not=\mbz)$, then the minority group is the set of samples with $\mby\not=\mbz$ because using only the nuisance results in predicting $\mby=b$ where $\mbz=b$.}.
Upweighting these samples produces a dataset with lesser dependence between the nuisance and the label.
Models learned on the upweighted data depend more on the semantics.
}
See \cref{alg:jtt} for pseudocode.

\begin{algorithm}[ht]
\begin{flushleft}
    \begin{algorithmic}
    \textbf{Input:} Training set $D$ and hyperparameters $T$ and $\lambda_{\text{up}}$.
    \textbf{Stage one: identification}
    
        1. Train identification model $f_\theta$ on $D$ via ERM for $T$ steps. 

        2. Construct the errors set of training examples misclassified by $f_\theta$.
    
    \textbf{Stage two: upweighting identified points}
        
        3. Construct upsampled dataset $D_\text{up}$ containing examples in the error set repeated $\lambda_{\text{up}}$ times and all other examples once.
        
        4. Train final model $f_\gamma$ on $D_\text{up}$ via ERM.
    \end{algorithmic}
    \end{flushleft}
    \caption{\Gls{jtt}.}
    \label{alg:jtt}
\end{algorithm}
In this work, we build the identification model on semantic corruptions i.e. we learn $f_\theta$ to predict $\mby$ from $T(\mbx)$.
The training samples to be upweighted are the ones misclassified when predicting with the identification model on semantic-corrupted versions of the sample, i.e. $T(\mbx)$.
The second stage is run as in \citep{liu2021just} with training data.

\paragraph{Optimization-generalization Dilemma}

Like many other algorithms in the \gls{ood} generalization literature, training \nmeth{s}s based on semantic corruptions may also suffer from obstacles due to optimization and generalization: employing statistical constraints to handle distribution shift may not build models that perform well OOD due to overfitting \citep{wald2022malign}, training difficulties 
\citep{chen2022pareto,zhang2022rich,chen2024understanding}, or reliance on inappropriate inductive biases \citep{nagarajan2020understanding,puli2023don}. Some approaches in the literature can alleviate these difficulties: two-stage methods incorporate the \gls{ood} objective only when training smaller models on top of large ones \citep{chen2022pareto,zhang2022rich,chen2024understanding,yong2023spurious,kirichenko2022last}, subsampling instead of weighting \citep{sagawa2020investigation,idrissi2022simple}, or large $\ell_2$ regularization \citep{sagawa2019distributionally}.

In our implementations we use validation data and regularization to tune parameters for the weighted-\gls{erm} algorithm as proposed in the original papers of the \nmeth{s} we experiment with.
As \gls{erm} is standard practice, there are no new optimization difficulties but generalization difficulties can occur due to overfitting \citep{wald2022malign,puli2023don}.
Any improvements in generalization in weighted-\gls{erm} will lead to improvements in models built by \nmeth{s} with biased models from semantic corruptions.

\section{Further experimental details}\label{appsec:exps}

\subsection{Remark on baseline corruptions}\label{appsec:remark-on-baselines}

\Gls{nurd} with the baseline corruption \textsc{gauss-noise} outperforms \gls{erm} and closes $80\%$ of the gap between \gls{erm} and known-$\mbz$ \gls{nurd} in \cref{tab:nurd-wb-results}.
We explain such an improvement as a consequence of \textsc{gauss-noise} corrupting semantics more than it corrupts nuisances; we explain below.
In tasks like waterbirds, nuisances are present in most if not all patches of the image regardless of where the patches appear.
On the other hand, semantic features are localized to a few adjacent patches (like the birds parts appearing next to each other).
When nuisances are present is many more patches than the semantics, adding gaussian noise to all pixels corrupts semantics more than nuisances.
To see why, consider meausurements of a quantity as a gaussian random variable with the quantity as its mean. More measurements lead to better estimates of the mean.

\begin{figure}[t]
\centering
    \includegraphics[width=0.97\textwidth]{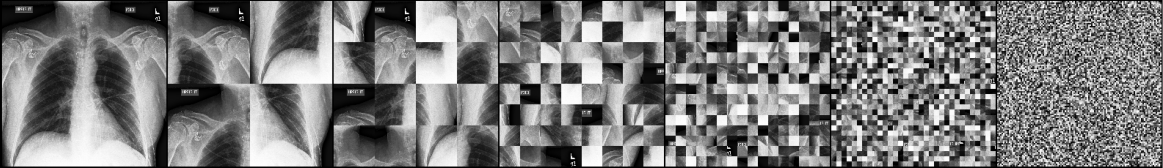}
    \caption{Example of \gls{pr} of a chest X-ray image. The image is followed by \glspl{pr} of size $112, 56, 28, 14, 7, 2$.
        }
    \label{fig:cardio-pr}
    \vspace{-10pt}
\end{figure}

\subsection{Implementation details}
\label{ssec:implementation}

Each experiment in the paper was run on up to 2 RTX8000 GPUs.
The hyperparameters for methods that use known nuisances in the training data, like \gls{nurd}, \gls{poe}, \gls{dfl} are tuned on validation data from the training distribution.
For \gls{nurd}, we select corruption hyperparameters using the mean of the balanced validation accuracy across $10$ seeds.
We do the same when using semantic corruptions.

\paragraph{Experimental details for Waterbirds}
For the \gls{nurd} setup, the training, validation, and test datasets have $ 3020, 756, 800$ samples respectively.
        We use a single architecture to parameterize the predictive model and the weight model in this experiment: two fully connected layers on top of a ResNet18 initialized at weights pretrained on Imagenet.
        We use the same training procedure for \gls{nurd} with known nuisances or with semantic corruptions. 
        Both models are trained with cross-entropy.
        The weight model is optimized with the default Adam optimizer for 20 epochs with a batch size of $64$.
        The predictive model is optimized with the Adam optimizer for 20 epochs with a learning rate of $0.0002$, a weight decay of $0.01$, and a batch size of $250$.

For the \gls{jtt} setup, the training, validation, and test datasets have $4795, 1199, 5794$ samples respectively.
For \gls{jtt}, we use the same model and model parameters as \citet{liu2021just} using their released code.
We repeat the details here for completeness.
The model for both stages of \gls{jtt} is a ResNet-50.
Both models are optimized by stochastic gradient descent (SGD) with momentum $0.9$, weight decay $1.0$, and learning rate $1\times 10^{-5}$.
Both models are trained for 300 epochs with batch size 64, using batch normalization and no data augmentation.
The identification model used to select samples to upweight corresponds to epoch $60$ and the upweighting constant is $\lambda=100$.

\paragraph{Experimental details for cardiomegaly detection.}
The training, validation, and test datasets are fixed across seeds and have $ 18000, 2000, 1000$ samples respectively.
To run reweighting-\gls{nurd}, we use a single architecture to parameterize the predictive model and the weight model in this experiment: two fully connected layers on top of a ResNet18 initialized at weights pretrained on Imagenet.
In known-nuisance \gls{nurd} with the hospital as the nuisance, the biased model is an estimate of $\ptr(\mby \g \text{hospital})$, which is obtained by binning the samples based on the hospital and averaging the labels.
 We use the same training procedure for \gls{nurd} with known nuisances or with semantic corruptions.
 Both weight and predictive models are trained with cross-entropy.
 The weight model and the predictive model are optimized with the Adam optimizer over $25$ epochs with a batch size of $256$, and learning rate $0.001$.

\paragraph{Implementation details for \gls{nli}}

For \gls{poe} and \gls{dfl}, we build classifiers by fine-tuning a pretrained BERT model ~\citep{Devlin2019BERTPO} on the data. We follow the same training procedure and hyperparameter details as used in ~\citet{mahabadi2019end} --- models were trained on the MNLI training dataset which consists of 392k examples, with a learning rate of $2 \times 10^{-5}$ with a batch size of 8 using the Adam Optimizer. All models are trained for 3 epochs. The development set contains 9815 examples and the HANS test contains 30000 examples. Since the HANS dataset has only two labels --- `entailment' and `non-entailment' --- we combine the neutral and contradiction classes during inference on HANS.

For the \gls{jtt} setup, \citet{liu2021just} mix the training and development sets from MNLI and create their own training, validation, and test sets of sizes $206175,  82462, 123712$ respectively.
For \gls{jtt}, we use the same model and model parameters as \citet{liu2021just} using their released code.
We use the optimal hyperparameters reported in \cite{liu2021just} for the learning rate, weight decay, and the upweighting constant.
We repeat the details here for completeness.
The model for both stages of \gls{jtt} is a pretrained BERT model that is finetuned during training.
Both models are optimized by the AdamW optimizer with clipping for the predictive model, no weight decay, and an initial learning rate of $2\times 10^{-5}$.
Both models are trained for $5$ epochs with batch size $32$ and dropout.
The identification model used to select samples to upweight corresponds to epoch $2$ for vanilla \gls{jtt} (reported optimal in \citet{liu2021just}); for \gls{jtt} with semantic corruption, we select one from $2,3$ using validation group annotations.
For both, the upweighting constant is $\lambda=6$.
Our runs with these parameters did not yield the test worst-group accuracy reported in \citep{liu2021just} ($72.6\%$); our experiments yielded a test worst-group accuracy $71.3\%$.
We expect this may be due to the differences in the random seed; \gls{jtt} is sensitive to hyperparameters and differences in order of batches may result in drops in performance.

In \gls{nr}, when the number of words in the sentence is not a multiple of $n$, there will be one $k$-gram ($k < n$).
In implementing \gls{nr}, we ensure that the position of this k-gram is randomized i.e. we make sure that it does not always occur at the end of the sentence, for example. \gls{nr} is implemented before word-piece tokenization (which BERT uses), to ensure that we randomize words instead of subwords.
We also create a small HANS-like development set, which is used to tune the size parameter. This set is constructed by randomly sampling $1000$ examples from the HANS training set, which has zero overlap with the HANS test set.

\subsection{Full results tables and additional experiments}\label{appsec:all-results}

We give the results for all size parameters; see \cref{tab:nurd-wb-results-full},  \cref{tab:nli-results-full}, \cref{tab:xray-results-full}, \cref{tab:appsec-jtt-wb-results}, and \cref{tab:appsec-jtt-nli-results}.
To report the same metrics as in \cite{mahabadi2019end} for \gls{poe} and \gls{dfl} and \cite{puli2021predictive} for \gls{nurd}, we report standard error for \gls{nurd} and standard deviation for \gls{poe} and \gls{dfl} .

\subsubsection{Results on Adversarial NLI \citep{nie2019adversarial} and CAD \citep{kaushik2019learning}}\label{appsec:added-evaluations}

\begin{wraptable}[18]{r}{0.3\textwidth}
\centering
\vspace{-14pt}
    \caption{
    Test worst-group (WG) accuracies of \gls{jtt} on modified waterbirds where the spurious correlation is weaker than the invariant relationship.
Corruption-powered \gls{jtt} outperforms \gls{erm}, vanilla \gls{jtt}, and \gls{jtt} with baseline corruptions
 (\textsc{rand-crop}, \textsc{gauss-noise}) by $\geq 4.4\%$.
}
    \label{tab:addexps-wb-results}
  \centering
\begin{tabular}{lc}
    \toprule
      Method &  test WG acc. \\
    \midrule
        \textit{Vanilla} \gls{jtt} 
        &  $78.6 \%$ 
\\
	\midrule
\gls{pr}          
        &  $ 84.6\%$ 
    \\
\gls{rm}          
        &  $ 85.2\%$ 
    \\
\gls{ff}          
        &  $ 83.2\%$ 
    \\
\gls{if}          
        &  $ 83.0\%$ 
    \\
	\midrule   
\textsc{rand-crop}       
        &  $ 76.2\%$ 
    \\
\textsc{gauss-noise}        
        &  $ 75.9\%$ 
    \\
	\midrule    
            \gls{erm} 
            & $76.1\%$
	\\
	\bottomrule
\end{tabular}
\end{wraptable}

In \cref{tab:anli_results} and \cref{tab:cad_results}, we report evaluations of \gls{poe} and \gls{dfl} models on the adversarial ANLI \citep{nie2019adversarial} and the counterfactually augmented dataset \citep{kaushik2019learning}.

\subsubsection{Additional experiments}

\paragraph{Experiments with weaker spurious correlations.}

To verify the effectiveness of the semantic corruptions for powering \nmeth{s} like \gls{jtt} that rely on assumptions on \gls{erm}-trained models, we experiment with a modified version of the Waterbirds dataset.
In the modified dataset, the spurious feature predicts the label only $75\%$ of the time; this is weaker than the $93\%$ in the original dataset and the invariant relationship which achieves $>85\%$ accuracy across all groups.
We ran \gls{erm}, \gls{jtt}, and corruption-powered \gls{jtt}. 
For both versions of \gls{jtt}, we tune over the same hyperparameters as in \citet{liu2021just}.
The results in \cref{tab:addexps-wb-results} show that corruption-powered \gls{jtt} is better than vanilla \gls{jtt} and \gls{erm}. 
The improvement of corruption-powered \gls{jtt} over vanilla \gls{jtt} increases from $0.5\%$ in \cref{tab:jtt-wb-results} to $4.4\%$ in \cref{tab:addexps-wb-results}; this indicates that vanilla \gls{jtt} is more sensitive to the strength of the spurious correlation than corruption-powered \gls{jtt}.

\begin{wraptable}[15]{r}{0.26\textwidth}
\vspace{-14pt}
\centering
    \caption{
   Accuracy of predicting the label from the image corrupted by \gls{pr} as patch-size decreases.
As the label is independent of the nuisance, a lower accuracy means that more semantic information is corrupted.
}
    \label{tab:addexps-y-pred}
  \centering
\begin{tabular}{lc}
    \toprule
\gls{pr} size &  Accuracy \\
    \midrule
Full image & $86\%$
\\
112 
        &  $76 \%$ 
\\
56  
        &  $ 73\%$ 
    \\
28     
        &  $ 64\%$ 
    \\
14   
        &  $ 58\%$ 
    \\
7   
        &  $ 57\%$ 
    \\
	\bottomrule   
\end{tabular}
\end{wraptable}

\paragraph{Experiments with multiple spurious features.}
We run \gls{rm}-powered \gls{nurd} with a modified version of the ColorFulMNIST dataset \citep{yong2023spurious}.
The images consist of $42\times 42\times 3$ pixels, with the middle $14\times 14$ forming the MNIST image showing a $0$ or a $1$ and the rest being background patches.
The digit in the middle predicts the binary label $1$ or $0$ with $75\%$ accuracy.
Given some $p\in [0,1]$, this dataset sets each of the background patch colors deterministically based on the image in the middle with probability $p$; with probability $1-p$,  each background is a random color (see figure 5 in \citep{yong2023spurious}.)
We generate the training data with $p=0.9$, and the validation and test data with $p=0$.

\Gls{rm}-powered \gls{nurd} with central-\gls{roi} sizes $14$ and $28$ achieves test accuracies $71.1\%$ and $70.3\%$ respectively, beating \gls{erm} which achieves $51.7\%$ because it relies more on the background colors.
\gls{pr} is not suited for this experiment because the different nuisance colors are chosen based on the patch position, and \gls{pr} randomizes patch positions which corrupt these nuisances. 

\paragraph{Experiments showing that corrupting the semantics is the reason behind the improved \gls{ood} performance in corruption-powered \nmeth{s}.}
First, we show that corruptions actually do corrupt semantics, taking \gls{pr} as the example.
We focus on the Waterbirds dataset to show how patch size affects semantics. 
For this investigation, we construct training and test datasets where the label and nuisance are independent and build models for predicting the label.

The results are in \cref{tab:addexps-y-pred} and show that as patch-size decreases, more semantic information is lost.
These results mean that for patch sizes $<28$, a biased model built from the corrupted image cannot predict the label well using semantics alone; the accuracy of random chance is $50\%$.
As the label is independent of the nuisance, a lower accuracy means more semantic information is corrupted.
However, on the original dataset, our biased models at these patch sizes achieve at least $85\%$ accuracy in predicting the label from the corrupted images, meaning that they rely mostly on the nuisance.

Second, to show that corruptions actually do help, we ran the full \gls{nurd} algorithm on the Waterbirds dataset from \citep{puli2021predictive} with a biased model built directly on the uncorrupted covariates; that is we train a model with \gls{erm} to predict $\mby$ from $\mbx$ and use it as the biased model in \gls{nurd}. 
The resulting test accuracy is $<70\%$. 
When using patch-sizes under $28$, the \gls{pr}-powered \gls{nurd} algorithm achieves a test accuracy of nearly $87\%$. 
This shows that the corruption of semantics is directly responsible for improving model robustness.

\begin{table}[ht]
\centering
 \vspace{20pt}
\begin{small}
 \caption{Mean and standard error of test accuracy across $10$ seeds of \gls{nurd} on classifying waterbirds. 
  \textit{Known}-nuisance \gls{nurd} uses a label for the type of background as the nuisance.
Selecting the size hyperparameter based on the average accuracy over $10$ seeds on the validation dataset gives $14$ for \gls{pr}, $196$ for \gls{rm}, $168$ for \gls{ff}, and $0.2$ for \gls{if}.
Consider the gap between \gls{erm} and known-nuisance \gls{nurd}.
\gls{nurd} with \gls{pr}, \gls{rm}, \gls{ff}, and \gls{if} close $99\%,99\%,82\%,99\%$ of the gap respectively.
\gls{nurd} with these semantic corruptions outperforms \gls{erm} and \gls{nurd} with \textsc{rand-crop}  and \textsc{gauss-noise}.
\gls{nurd} with all semantic corruptions outperforms \gls{erm} ($69.2\%$).
  }
    \label{tab:nurd-wb-results-full}
     \centering
    \begin{tabular}{cccccccccccc}
    \toprule 
      & \textit{known} 
      & \textsc{rm}
      & \textsc{rm}
      & \textsc{rm}
      & \textsc{rm}
      & \textsc{pr}
      & \textsc{pr}
      & \textsc{pr}
      & \textsc{pr}
	& 
\\
	& $\mbz$
      & 196
      & 168
      & 140
      & 112
      & 7
      & 14
      &  28
      & 56
	& 
	\gls{erm}
\\
\midrule
      Mean
      & 
      $87.2 \%$ 
      & ${86.9} \%$ 
      & ${86.6} \%$ 
      & $86.2 \%$ 
      & $86.3 \%$ 
      &  $85.6 \%$ 
      &  $86.9 \%$ 
      &  $82.5 \%$ 
      &  $84.9 \%$ 
      &  $68.0 \%$ 
\\
      Std. err.
      & $1.0\%$ 
      & $1.1\%$ 
      & $1.2\%$ 
      & $1.8\%$ 
      & $1.6\%$ 
      &  $1.4\%$ 
      &  $1.2\%$ 
      &  $2.0\%$ 
      &  $1.4\%$ 
      &  $1.9\%$ 
\\
\midrule
\\
  & 
    
      & \textsc{ff}
      & \textsc{ff}
      & \textsc{ff}
      & \textsc{ff}
      & \textsc{if}
      & \textsc{if}
      & \textsc{if}
      & \textsc{if}
	& 
\\
       & 
      & $196$
      & $168$
      & $140$
      & $112$
      & $0.1$
      & $0.2$
      & $0.3$
      & $0.4$
	& 
\\
\midrule
      Mean
      & 
      & $83.8 \%$ 
      & $83.5\%$ 
      & $81.0 \%$ 
      & $80.3 \%$ 
      &  $81.2 \%$ 
      &  $86.9 \%$ 
      &  $85.0\%$ 
      &  $81.9 \%$ 
      & 
\\
      Std. err.
      & 
      & $1.2\%$ 
      & $1.1\%$ 
      & $1.4\%$ 
      & $1.7\%$ 
      &  $1.7\%$ 
      &  $1.1\%$ 
      &  $1.5\%$ 
      &  $1.7\%$
      & 
\\
\midrule
\\
  & 
      & \textsc{rand-crop}
      & 
      & 
      & 
      & \textsc{gauss}
      & \textsc{gauss}
      & \textsc{gauss}
      & \textsc{gauss}
	& 
\\
       & 
      & 
      & 
      &
      & 
      & $0.01$
      & $0.25$
      & $1$
      & $4$
	& 
\\
\midrule
      Mean
      & 
      & $73.7 \%$ 
      & 
      & 
      &
      &  $75.8 \%$ 
      &  $74.1 \%$ 
      &  $78.0 \%$ 
      &  $83.9 \%$ 
      & 
\\
      Std. err.
      & 
      & $2.0\%$ 
      & 
      & 
      & 
      &  $3.2\%$ 
      &  $3.1\%$ 
      &  $3.4\%$ 
      &  $1.4\%$
      & 
\\
	\bottomrule\\
\end{tabular}
\end{small}
\end{table}

\begin{table}[H]
\centering
\begin{small}	
 \caption{Average accuracies and standard deviation over $4$ seeds of \gls{poe} and \gls{dfl} with semantic corruptions on the HANS dataset.
The results for  \textit{known} \gls{poe} and \gls{dfl} from \cite{mahabadi2019end}, where both methods use known nuisances.
For both methods, selecting the size hyperparameter based on the average accuracy on a small dataset ($1000$ samples) from the test distribution gives $n=3$.
With this size, \gls{poe} with \gls{nr} performs better than known-nuisance \gls{poe} while \gls{dfl} with \gls{nr} closes ${84}\%$ of the gap between \gls{erm} and known-$\mbz$ \gls{dfl} .
 }
    \label{tab:nli-results-full}
  \centering
    \begin{tabular}{ccc}
    \toprule
      $\mbz$ &  \gls{poe} &   \gls{dfl} \\
    \midrule
        \textit{Known} 
        &  $66.3 \pm 0.6 \%$ 
        & $69.3\pm 0.2\%$   
    \\
        1-gram 
        &  $65.7 \pm 2.0 \%$ 
        & $66.5\pm 1.5\%$   
    \\
        2-gram  
        &  $ 66.0 \pm 0.9 \%$ 
        & $68.5\pm 0.7\%$   
    \\
        3-gram   
        &  $ 66.7 \pm 1.5  \%$ 
        & $68.4 \pm 1.5\%$   
    \\
        4-gram   
        &  $ 66.2 \pm 2.9 \%$ 
        & $65.0 \pm 2.0\%$   
    \\
	\midrule    
            \gls{erm} 
            & $-$ 
            &  $63.6\%$.
	\\
	\bottomrule
\end{tabular}
\end{small}
\end{table}

\begin{table}[t]
\centering
\begin{small}
    \caption{Mean and standard error of test accuracy across $10$ seeds of \gls{nurd} on detecting cardiomegaly from chest X-rays.
\textit{Known}-nuisance \gls{nurd} uses the hospital as the nuisance.
Selecting the corruption parameters based on the mean accuracy over $10$ seeds on the validation dataset gives $14$ for \gls{pr}, $196$ for \gls{rm}, $168$ for \gls{ff}, and $0.1$ for the \gls{if}.
 Consider the gap between \gls{erm} and known-nuisance \gls{nurd}.
\gls{nurd} with \gls{pr}, \gls{rm}, \gls{ff}, and \gls{if} close $72\%,82\%,65\%,35\%$ of the gap respectively.
\gls{nurd} with semantic corruptions outperforms \gls{nurd} with baseline augmentations \textsc{rand-crop} and \textsc{gauss-noise}.
\gls{nurd} with \gls{pr} and \gls{rm} outperforms \gls{erm} for all size parameters.
}
    \label{tab:xray-results-full}
  \centering
    \begin{tabular}{ccccccccccc}
    \toprule 
      & 
      \textit{known} 
      & \textsc{rm}
      & \textsc{rm}
      & \textsc{rm}
      & \textsc{rm}
      & \textsc{pr}
      & \textsc{pr}
      & \textsc{pr}
      & \textsc{pr}
	& 
\\
       &  $\mbz$
      & 196
      & 168
      & 140
      & 112
      & 7
      & 14
      &  28
      & 56
	& 
	\gls{erm}
\\
\midrule
      Mean
      & 
      $81.7 \%$ 
      & $78.7 \%$ 
      & $78.3 \%$ 
      & $77.2 \%$ 
      & $73.6 \%$ 
      &  $76.2 \%$ 
      &  $77.0 \%$ 
      &  $74.9 \%$ 
      &  $74.3 \%$ 
      &  $65.3 \%$
\\
      Std. err.
      & $0.3\%$ 
          & $0.3\%$ 
          & $0.8\%$ 
          & $0.8\%$ 
          & $0.7\%$ 
      &  $1.2\%$ 
      &  $1.2\%$ 
      &  $1.0\%$ 
      &  $1.4\%$
        &  $1.1\%$
\\
\midrule
\\
  & 
    
      & \textsc{ff}
      & \textsc{ff}
      & \textsc{ff}
      & \textsc{ff}
      & \textsc{if}
      & \textsc{if}
      & \textsc{if}
      & \textsc{if}
	& 
\\
       & 
      & $196$
      & $168$
      & $140$
      & $112$
      & $0.1$
      & $0.2$
      & $0.3$
      & $0.4$
	& 
\\
\midrule
      Mean
      & 
      & $74.4 \%$ 
      & $76.0\%$ 
      & $75.3 \%$ 
      & $71.3 \%$ 
      &  $71.0 \%$ 
      &  $68.0 \%$ 
      &  $62.0 \%$ 
      &  $57.1 \%$ 
      & 
\\
      Std. err.
      & 
      & $1.5\%$ 
      & $0.6\%$ 
      & $0.9\%$ 
      & $1.6\%$ 
      &  $1.0\%$ 
      &  $1.6\%$ 
      &  $1.8\%$ 
      &  $3.2\%$
      & 
\\
\midrule
\\
  & 
      & \textsc{rand-crop}
      & 
      & 
      & 
      & \textsc{gauss}
      & \textsc{gauss}
      & \textsc{gauss}
      & \textsc{gauss}
	& 
\\
       & 
      &
      & 
      & 
      & 
      & $0.01$
      & $0.25$
      & $1$
      & $4$
	& 
\\
\midrule
      Mean
      & 
      & $59.9 \%$ 
      &
      & 
      & 
      &  
      $62.3 \%$ 
      &  $63.5 \%$ 
      &  $68.0 \%$ 
      &  $69.0 \%$ 
      & 
\\
      Std. err.
      & 
      & 
			$2.1\%$ 
      & 
      & 
      & 
      &  $3.7\%$ 
      &  $3.4\%$ 
      &  $1.1\%$ 
      &  $1.9\%$
      & 
\\
	\bottomrule\\
\end{tabular}
\end{small}
\end{table}

\begin{table}[t]
\centering
\vspace{-15pt}
    \caption{
    Test worst-group accuracies of \gls{jtt} with semantic corruptions on waterbirds.
Selecting the corruption hyperparameters on the validation worst-group accuracy gives size $14$ for \gls{pr}, size $196$ for \gls{rm}, size $112$ for \gls{ff}, and threshold $0.4$ for \gls{if}.
\gls{jtt} with these semantic corruptions outperforms \gls{erm}, vanilla \gls{jtt}, and \gls{jtt} with the baseline corruptions \textsc{rand-crop} and \textsc{gauss-noise}.
\gls{jtt} with \gls{pr} and \gls{rm} outperforms \gls{jtt} with the baseline corruptions and \gls{erm} for all sizes.}
    \label{tab:appsec-jtt-wb-results}
  \centering
    \begin{tabular}{ccccccccccc}
    \toprule 
      \textit{Vanilla} 
      & \textsc{rm}
      & \textsc{rm}
      & \textsc{rm}
      & \textsc{rm}
      & \textsc{pr}
      & \textsc{pr}
      & \textsc{pr}
      & \textsc{pr}
	& 
\\
	\gls{jtt}
      & 196
      & 168
      & 140
      & 112
      & 7
      & 14
      & 28
      & 56
	& 
	\gls{erm}
\\
\midrule
      $86.5 \%$ 
      & ${88.2} \%$ 
      & ${88.0} \%$ 
      & $86.9 \%$ 
      & $86.2 \%$ 
      &  ${89.3} \%$ 
      &  ${89.0} \%$ 
      &  ${88.9} \%$ 
      &  ${89.1} \%$ 
      &  $72 \%$ 
\\
\midrule
\\
      & \textsc{ff}
      & \textsc{ff}
      & \textsc{ff}
      & \textsc{ff}
      & \textsc{if}
      & \textsc{if}
      & \textsc{if}
      & \textsc{if}
	& 
\\
      & $196$
      & $168$
      & $140$
      & $112$
      & $0.1$
      & $0.2$
      & $0.3$
      & $0.4$
	& 
\\
\midrule
      & $82.5 \%$ 
      & $84.5\%$ 
      & $85.2 \%$ 
      & $87.2 \%$ 
      &  $69.1 \%$ 
      &  $80.0 \%$ 
      &  $81.7\%$ 
      &  $87.0 \%$ 
      & 
\\
\midrule
\\
      & \textsc{rand-crop}
      & 
      & 
      & 
      & \textsc{gauss}
      & \textsc{gauss}
      & \textsc{gauss}
      & \textsc{gauss}
	& 
\\
      & 
      & 
      &
      & 
      & $0.01$
      & $0.25$
      & $1$
      & $4$
	& 
\\
\midrule
      & $75\%$
      & 
      & 
      & 
      &  $0.0 \%$ 
      &  $0.0 \%$ 
      &  $71.0 \%$ 
      &  $0.0 \%$ 
      &
\\
	\bottomrule\\
\end{tabular}
\end{table}

\begin{table}[t]
\centering
\begin{small}
\caption{Worst-group and average test accuracies of \gls{jtt} with semantic corruptions on \gls{nli}.
\gls{jtt} with \gls{pm} and \gls{nr} of every size outperforms vanilla \gls{jtt}.
Selecting the size hyperparameter for \gls{nr} using validation worst-group accuracy, like \citet{liu2021just} do for vanilla \gls{jtt}, gives $n=1$. At this size, \gls{jtt} with \gls{nr} outperforms vanilla \gls{jtt} by $3\%$ accuracy.
  }
    \label{tab:appsec-jtt-nli-results}
  \centering
    \begin{tabular}{ccc}
    \\
    \toprule
	 &  Worst-group &   Average \\
    \midrule
        \textit{Vanilla} \gls{jtt} &  $71.3 \%$ & $79.1\%$   
	\\
        \gls{pm}  &  ${72.1}\%$ & $79.9
        \%$   
    \\
        1-gram  &  ${74.3}\%$ & $79.7\%$   
    \\
        2-gram   &  ${71.9} \%$ & $80.0\%$   
    \\
        3-gram   &  ${72.0} \%$ & $80.1\%$   
    \\
        4-gram   &  ${73.4}\%$ & $80.4\%$   
    \\
	\midrule    
            \gls{erm} & $67.9\%$ &  $-$
	\\
	\bottomrule\end{tabular}

\end{small}
\end{table}

\begin{table}[t]
\centering
\caption{ANLI \citep{nie2019adversarial} evaluations of models trained on MultiNLI.
With a t-test to measure statistical significance, at the standard significance level of 0.05, we found that \gls{poe} with \gls{nr} gave a statistically significant improvement over the baseline on ANLI-R1 and ANLI-R2, while \gls{dfl} gave a statistically significant improvement on ANLI-R1.}
\label{tab:anli_results}
\begin{tabular}{cccc}
\\
\toprule
Model & ANLI - R1 & ANLI - R2 & ANLI - R3 \\
\midrule
\gls{erm} & $23.1 \pm 0.9$ & $28.2 \pm 0.8$ & $29.8 \pm 0.4$ \\
\gls{poe}-known & $23.5 \pm 0.6$ & $27.8 \pm 0.8$ & $29.8 \pm 0.8$ \\
\gls{dfl}-known & $23.7 \pm 1.3$ & $27.8 \pm 1.1$ & $30.4 \pm 0.9$ \\
\gls{poe} - n3 & $24.8 \pm 1.1$ & $29.2 \pm 0.4$ & $30.4 \pm 1.2 $ \\
\gls{dfl} - n3 & $24.9  \pm 0.6$ & $29.0 \pm 1.2$ & $29.9 \pm 0.3$ \\
\gls{poe} - \gls{pm} & $23.6 \pm 1.2$ & $27.3 \pm 0.8$ & $29.8 \pm 0.8$ \\
\gls{dfl} - \gls{pm} & $ 22.3 \pm 0.7$ & $ 27.7 \pm 0.6$ & $29.3 \pm 1.1$ \\
\bottomrule
\end{tabular}
\end{table}

\begin{table}[t]
\centering
\caption{Mean and standard deviation of CAD \citep{kaushik2019learning} test accuracy over 4 seeds.
At the end, we also report the results of finetuning BERT on CAD training data from \citep{kaushik2019learning}.
When trained on MNLI, on average over the CAD subsets \textsc{RH} and \textsc{RH}, \gls{dfl} and \gls{poe} with semantic corruptions,  \gls{dfl} and \gls{poe} with known-nuisances, and \gls{erm} perform on par (within one std.) or better than finetuning directly on the training CAD dataset.
The improvement over finetuning directly on CAD may be due to the fact that the CAD dataset  is much smaller than MNLI ($~7k$ vs. $~400k$).
}
\label{tab:cad_results}
\begin{tabular}{lccc}
\\
\toprule
Method & \textsc{RP} & \textsc{RH} &  Avg. on \textsc{RP} and \textsc{RH} \\
\midrule
\gls{erm} on MNLI & $61.1 \pm 0.3$ &  $76.5 \pm 0.4$ &  $68.8 \pm 0.2$ \\
\midrule
\gls{poe}-known & $60.6 \pm 0.5$ &  $77.0 \pm 1.1$ &  $68.8 \pm 0.3$ \\
\gls{poe} 3-gram & $60.8 \pm 0.5$ &  $76.1 \pm 0.7$ &  $68.4 \pm 0.2$ \\
\gls{poe} \gls{pm} & $61.7\pm 0.6$ &  $75.6 \pm 1.0$ &  $68.6 \pm 0.5$ \\
\midrule
\gls{dfl}-known & $60.6 \pm 0.8$ &  $76.2 \pm 0.7$ &  $68.4 \pm 0.4$ \\
\gls{dfl} 3-gram & $58.4 \pm 1.8$ &  $72.7 \pm 1.0$ &  $65.5 \pm 1.4$ \\
\gls{dfl} \gls{pm} & $62.4 \pm 0.7$ &  $76.1 \pm 0.8$ &  $69.3\pm 0.6$ \\
\midrule
\gls{erm} on CAD (from \citep{kaushik2019learning}) & $64.6 $   & $67.8 $  & $66.2 $  \\
\bottomrule
\end{tabular}
\end{table}

\end{document}